\pdfoutput=1
\documentclass[letterpaper]{article} 
\usepackage[submission]{aaai25}  

\usepackage{times}  
\usepackage{helvet}  
\usepackage{courier}  
\usepackage[hyphens]{url}  
\usepackage{graphicx} 
\usepackage{caption} 
\usepackage{soul}
\usepackage{amsmath}
\usepackage{amsfonts}
\usepackage{amsthm}
\usepackage{amssymb}
\usepackage{booktabs}

\usepackage{algorithm}
\usepackage{algorithmicx}
\usepackage[noend]{algpseudocode}
\usepackage{comment}
\usepackage{todonotes}

\usepackage{color, colortbl, etoolbox}
\usepackage{tabulary, array}

\usepackage{comment}

\usepackage{cancel}

\usepackage{array,multirow,makecell}

\makeatletter
\patchcmd{\Gin@setfile}
{\ProvidesFile}
{\xdef\imgfilename{#3}\ProvidesFile}
{}{}
\makeatother

\newtheorem{example}{Example}

\newtheorem{notation}{Notation}
\newtheorem{definition}{Definition}

\newtheorem{proposition}{Proposition}

\newcommand{\nci}{| \! \! \! \sim}

\newcommand{\cW}{{\mathcal{W}}}

\newcommand{\bL}{{\mathbf{L}}} 
\newcommand{\bM}{{\mathbf{M}}}

\newcommand{\la}{\langle}
\newcommand{\ra}{\rangle}

\newcommand{\lb}{\lbrace}
\newcommand{\rb}{\rbrace}
\newcommand{\m}{\mathtt}

\newcommand{\abso}[1]{\m{abs}(#1)}
\newcommand{\tuple}[1]{\langle #1 \rangle}

\newcommand{\quality}{\mathtt{Q}}

\newcommand{\fCN}[1]{\mathtt{fCn}_{#1}}

\newcommand{\card}[1]{{|#1|}}

\newcommand{\wLog}{\m{wLog}}

\newcommand{\qual}{{\m{Q}}}

\newcommand{\Arg}{{\sf Arg}}

\newcommand{\aArg}{{\sf aArg}}

\newcommand{\NNF}{{\mathtt{NNF}}}
\newcommand{\lit}{{\mathtt{Lit}}}

\newcommand{\Flat}{\mathtt{Flat}}
\newcommand{\Wei}{\mathtt{Weight}}

\newcommand{\E}{C}
\newcommand{\F}{F}
\newcommand{\len}{\mathtt{Card}}
\newcommand{\dnf}{\mathtt{Dnf}}
\newcommand{\lan}{\mathtt{Lan}}
\newcommand{\cnf}{\mathtt{Cnf}}
\newcommand{\fa}{\forall \:}
\newcommand{\ex}{\exists \:}
\newcommand{\ato}{\mathtt{A}}
\newcommand{\enua}{\mathtt{Ena}}

\newcommand{\enui}{\mathtt{Eni}}
\newcommand{\ei}{\mathtt{I}}

\newcommand{\fo}{\mathtt{f}}

\newcommand{\nf}{\mathtt{Nf}}
\newcommand{\dn}{\mathtt{Dn}}
\newcommand{\Inc}{\mathtt{Inc}}
\newcommand{\wLan}{\mathtt{wLan}}

\newcommand{\psco}[1]{{\mathtt{M}^\mathtt{psc}_{#1}}}
\newcommand{\pwco}[1]{{\mathtt{M}^\mathtt{pwc}_{#1}}}
\newcommand{\dsco}{{\mathtt{M}^\mathtt{dsc}}}
\newcommand{\dwco}{{\mathtt{M}^\mathtt{dwc}}}

\newcommand{\ppi}[1]{{\mathtt{M}_{#1}^\mathtt{ppi}}}
\newcommand{\dpi}[1]{{\mathtt{M}_{#1}^\mathtt{dpi}}}

\newcommand{\cinf}{\mathtt{Inf}}
\newcommand{\fdn}{\mathtt{fDn}}
\newcommand{\cmmin}[1]{{\mathtt{M}_{#1}^\mathtt{dm}}}
\newcommand{\cmpen}[1]{{\mathtt{M}_{#1}^\mathtt{pm}}}
\newcommand{\mymin}{\mathtt{Min}}
\newcommand{\mymax}{\mathtt{Max}}
\newcommand{\cmtve}[1]{{\mathtt{M}_{#1}^\mathtt{tv}}}
\newcommand{\cmtvetve}[1]{{\mathtt{M}_{#1}^\mathtt{tp}}}
\newcommand{\cmbl}[1]{{\mathtt{M}_{#1}^\mathtt{bp}}}
\newcommand{\cmcd}[1]{{\mathtt{M}_{#1}^\mathtt{cd}}}
\newcommand{\cmcp}[1]{{\mathtt{M}_{#1}^\mathtt{cp}}}
\newcommand{\cmdg}[1]{{\mathtt{M}_{#1}^\mathtt{dg}}}
\newcommand{\cmpg}[1]{{\mathtt{M}_{#1}^\mathtt{pg}}}
\newcommand{\cmsd}[1]{{\mathtt{M}_{#1}^\mathtt{sd}}}
\newcommand{\cmld}[1]{{\mathtt{M}_{#1}^\mathtt{ld}}}

\newcommand{\ent}{\mathtt{Enth}}

\definecolor{darkgreen}{rgb}{0.45, 0.75,0.5}
\definecolor{lightblue}{rgb}{0.15, 0.48, 0.75}

\title{An Axiomatic Study of the Evaluation of Enthymeme Decoding in\\ Weighted Structured Argumentation}

\author {
    Jonathan Ben-Naim\textsuperscript{\rm 1}, 
    Victor David\textsuperscript{\rm 2}\footnote{corresponding author}, 
    Anthony Hunter\textsuperscript{\rm 3}
}
\affiliations {
    \textsuperscript{\rm 1}Université Paul Sabatier, CNRS, IRIT, France\\
    \textsuperscript{\rm 2}Université Côte d'Azur, INRIA, CNRS, I3S, France\\
    \textsuperscript{\rm 3}University College London, UK\\
    jonathan.ben-naim@irit.fr,
    victor.david@inria.fr,
    anthony.hunter@ucl.ac.uk 
}

\begin{document}




\maketitle

    
\begin{abstract}




\vspace{-0.1cm}

An argument can be seen as a pair consisting of a set of premises and a claim supported by them. Arguments used by humans are often enthymemes, i.e., some premises are implicit.
To better understand, evaluate, and compare enthymemes, it is essential to decode them, i.e., to find the missing premisses. 
Many enthymeme decodings are possible.
We need to distinguish between reasonable decodings and unreasonable ones. 
However, there is currently no research in the literature on ``{\it How to evaluate decodings?}". 
To pave the way and achieve this goal,  
we introduce seven criteria related to decoding, based on different research areas.
Then, we introduce the notion of {\it criterion measure}, the objective of which is to evaluate a decoding with regard to a certain criterion. 
Since such measures need to be validated, we introduce several desirable properties for them, called {\it axioms}.
Another main contribution of the paper is the construction of certain criterion measures that are validated by our axioms.
Such measures can be used to identify the best enthymemes decodings.

\end{abstract}

\vspace{-0.4cm}
\section{Introduction}



In the literature on logic-based argumentation, a deductive argument is usually defined as a premise-claim pair where the claim is inferred (according to a logic) from the premises. 
However, when studying human debates (i.e. real world argumentation), it is common to find incomplete arguments, called enthymemes, for which the premises are insufficient for implying the claim. 
The reason for this incompleteness is varied, for example it may result from imprecision or error, e.g. a human may argue without knowing all the necessary information, or it may be intentional, e.g. one may presuppose that some information is commonly known and therefore does not need to be stated, or the employment of enthymemes is an instrument well known since Aristotle \cite{faure2010rhetoric} as one of the most effective in rhetoric and persuasion when it comes to interacting with an audience.


There are studies in the literature on understanding enthymemes in argumentation, using natural language processing 
\cite{habernal2017argument,singh2022irac,wei2022implicit}, 
but these do not identify logic-based arguments. 
There are also symbolic approaches for decoding enthymemes in structured argumentation including
\cite{hunter2007real,Dupin2011,black2012relevance,Hosseini2014,Xydis2020,Panisson2022,hunter2022understanding,Leiva2023,ben2024understanding},
but they only consider the task as identifying a set of formulae that could be added to the incomplete premises in order to entail the claim. This offers potentially many decodings, and there is currently a lack of means for comparing these decoding candidates.



In real-world argumentation, it is important to note that decoding is more general than that of completion. In fact, when we decode, we may add and subtract information, to obtain the most appropriate decoding.
Furthermore, given that several decodings of an enthymeme can be proposed, we then have the question of how to ``how to evaluate the quality of a candidate for decoding an enthymeme" in order to make an optimal choice of decoding.

Let us take the following example (which will be part of our running example) to illustrate an enthymeme with two possible decodings. 

\begin{itemize}

\item Enthymeme $E$: Knowing that Bob is wealthy, he is a researcher, he makes people happy, and he has people around him who seem to love him, then Bob is happy.
\vspace{0.2em}

\item Decoding $D_1$: Bob is a researcher and researchers are generally happy, so Bob is happy.
\vspace{0.2em}

\item Decoding $D_2$: Bob makes people happy and is surrounded by people who love him, and because giving and receiving love often makes people happy, Bob is happy.


\end{itemize}


To study whether $D_1$ or $D_2$ is a better decoding for $E$, we will represent knowledge by weighted logics, then we will propose quality measures based on measuring different aspects of a candidate for decoding (criterion measures). Given that the number of criterion measures for a criterion is infinite, we adopt an axiomatic approach, defining the constraints of a good measure.

\section{Weighted Logics}\label{sec:logic}

In the present section, we introduce the logic in which we represent enthymeme. Let us begin with the language. We chose a weighted one, because weights play an important role in enthymeme decodings as we will see it in the section devoted to the axioms.


\begin{definition}
    \upshape{
    A \textbf{weighted language} is a set $\cW$ such that:
    \begin{itemize}
        \item every element of $\cW$ is a pair of the form $\alpha = \tuple{f,w}$ such that $f$ is a \textit{formula} and $w$ a \textit{weight} in $[0,1]$;
        \item if $\tuple{f,w} \in \cW$, then, $\forall~v \in [0,1]$, $\tuple{f,v} \in \cW$;
        \item $\forall~w \in [0,1]$, $\tuple{\bot,w} \in \cW$ ($\bot$ means \textit{contradiction}).
    \end{itemize}

    

    
   }
\end{definition}



In this paper, we interpret the weights as confidence scores, i.e. a value representing confidence in the reliability of the formula. 
Thanks to the knowledge graph community, it is possible to obtain formulae in this weighted structure with a confidence score. 
Some graphs already have this kind of formulae \cite{chen2019embedding,david2023neomapy}, 
but it is interesting to note that there are also methods for learning them, such as AMIE+ \cite{galarraga2015fast}, RLvLR \cite{omran2019embedding}, or the reinforcement learning system guided by a value function \cite{chen2022rule}.

We are ready to introduce the notion of {\it weighted logic}.

\begin{definition} \label{def:wlogic}
    \upshape{
    A \textbf{weighted logic} is a triple $\bL {= \tuple{\cW,{\nci},t}}$ s.t.:
    \begin{itemize}
        \item $\cW$ is a weighted language;
        \item ${\nci}$ is a {\bf weighted consequence relation} on $\cW$, i.e., a relation from $2^{\cW}$ to $\cW$;
        \item $t$ is a \textit{consistency threshold} belonging to $[0,1]$.
    \end{itemize}

    
    We say that $\Gamma \subseteq \cW$ is \textbf{inconsistent} on $\bL$ iff there exists $w \geq t$ s.t. $\Gamma ~{\nci}~ \tuple{\bot,w}$, and it is denoted by $\Inc_\bL$ the set of all inconsistent set of formulae in $\bL$, and when $\bL$ is clear we will use only $\Inc$. 
    Otherwise, $\Gamma$ is said to be consistent.
    }
    
\end{definition}

Next, our goal is to present an instance of weighted logic that will be used in examples.

As a preliminary, we need two operators that extract the flat formulae or the weights from weighted formulae.

\begin{definition}
    \upshape{
    Let $\cW$ be a weighted language and $\Gamma \subseteq \cW$. We denote by $\Flat(\Gamma)$ the set of every \textbf{flat formula} appearing in $\Gamma$, i.e., $\Flat(\Gamma) = \{f : \exists~w, \tuple{f,w} \in \Gamma\}$. 

    We denote by $\Wei(\Gamma)$ the set of every \textbf{weight} appearing in $\Gamma$, i.e., $\Wei(\Gamma) = \{w : \exists~f, \tuple{f,w} \in \Gamma\}$.  

    
    }
\end{definition}

In the rest of the article, for any function taking a set of weighted formulae as a parameter, we will simplify the notation for the case of a single formula, e.g., for $\alpha \in \cW$, instead of writing $\Flat(\{\alpha\})$ we will simply write $\Flat(\alpha)$.


As another preliminary, we recall the notion of classical propositional language.

\begin{definition}
    \upshape{
    We denote by $\m{Lan}$ the set of every \textbf{classical propositional formula} built up from a given non-empty finite set of atomic formulae, denoted by $\m{A}$, and the usual connectives $\neg$, $\vee$, $\wedge$, $\rightarrow$, and $\leftrightarrow$.   
    A \textbf{literal} is either an element of $\m{A}$ or the negation of it, we denote the set of all literal by $\m{L}$. For any flat formula $f \in \m{Lan}$ we denote by $\lit(f)$ the set of literals occurring in $f$, and $\forall F \subseteq \m{Lan}$, $\lit(F) = \{l : l \in \lit(f) \text{ and } f \in F\}$.
    }
    \end{definition}

We are ready to introduce our specific weighted logic that we will be used in examples.

\begin{definition}
    \upshape{
    We denote by $\wLan$ the {\bf weighted propositional language}, i.e., $\wLan$ is the set of every pair $\tuple{f,w}$ such that $f$ in $\lan$ and $w \in [0,1]$. 
    
    We denote by $\wLog$ the \textbf{weighted propositional logic}, i.e., $\wLog$ is the triple $\tuple{\cW,{\nci},t}$ s.t. the following holds:
    \begin{itemize}
        \item $\cW = \wLan$;
        \item $\fa \Gamma \subseteq \wLan$, $\fa \alpha = \tuple{f,w} \in \wLan$,
        $\Gamma ~{\nci}~ \alpha$ iff $\big( f$ is a tautology and $w = 1 \big)$
        or $\big( f$ is not a tautology, $f$ classically follows from $\Flat(\Gamma)$, i.e. $\Flat(\Gamma) \vdash f$, and $w = \m{min}[\Wei(\Gamma)] \big)$;
        \item $t = 0.5$.
    \end{itemize}
    }

\end{definition}

Following examples 1 and 2 illustrate this definition. 
From now on, whenever we work with a weighted logic $\bL$, the typical instance we have in mind is $\wLog$.

\subsection{Normalization Methods}

Later in the paper, we count the number of elements in a set of formulae $\Gamma$. Thus, we need first to normalize the syntactic form of $\Gamma$. To achieve this goal, we propose the notion of {\it normalization method}. 

\begin{definition}
\upshape{
Let $\cW$ be a weighted language.
A {\bf normalization method} on $\cW$ is a function $N$ that normalizes the syntactic form of the formulae, i.e., $N$ is a function from $2^\cW$ to $2^\cW$.
}
\end{definition}

The rest of the present section is devoted to the construction of a specific normalization method on $\wLog$ that will be used in examples.

Our proposal is an alternative to the notion of {\it compilation} introduced in \cite{amgoud2021similarity} for propositional logic-based arguments.

For this, we need to capture classical interpretation with formula. 

\begin{definition}
\upshape{
We assume an enumeration (without repetition) $\enua = \la a_1, a_2, \ldots, a_n \ra$ of $\ato$, as well as an enumeration $\enui = \la \ei_1, \ei_2, \ldots, \ei_\m{m} \ra$ of the classical interpretations of $\lan$.

Next, let $i \in \lb 1, 2, \ldots, \m{m} \rb$. We denote by $\fo_i$ the formula {\bf representing} the interpretation $\ei_i$, i.e., $\fo_i$ is the conjunction of literals $l_1 \wedge l_2 \wedge \cdots \wedge l_r$ of $\lan$ such that $r = n$ and $\fa j \in \lb 1, \ldots, n \rb$, the following holds: $l_j = a_j$, if $a_j$ is true in $\ei_i$; $l_j = \neg a_j$, otherwise.
}
\end{definition}

We are ready to normalize the syntactic form of a propositional formula in a standard way.

\begin{definition}
\upshape{
Let $f \in \lan$. 
We denote by $\dnf(f)$ the {\bf canonical disjunctive normal form} of $f$, i.e.,
\begin{center}
   $\dnf(f) = \bigvee_{\lb i : \ei_i \textrm{ is a model of } f \rb} \fo_i.$
\end{center}

Next, we denote by $\cnf(f)$ the {\bf canonical conjunctive normal form} of $f$, i.e., $\cnf(f)$ is obtained from $\neg \dnf(\neg f)$ by, first, applying the De Morgan laws and double negation until we get a formula in CNF, and second iteratively applying the following three points:
\begin{enumerate}
    \item identify any two clauses $c = l_1 \vee l_2 \vee \cdots \vee l_n$ and $c' = l'_1 \vee l'_2 \vee \cdots \vee l'_m$ such that $n = m$ and, for some $i \in \lb 1, \ldots, n \rb$, for some $j \in \lb 1, \ldots, n \rb$, we have that
    $\big( l_i = \neg l'_j$ or $l'_j = \neg l_i \big)$ and
    $\la l_1, \ldots, l_{i-1}, l_{i+1}, \ldots, l_n\ra$ is a permutation of $\la l'_1, \ldots, l'_{j-1}, l'_{j+1}, \ldots, l'_n \ra$;
    \item remove $c'$ (unless $c'$ is a literal);
    \item remove $l_i$ from $c$ (unless $c$ is a literal).
\end{enumerate}
}
\end{definition}

Let us illustrate syntactic normalization.

\begin{example}
\upshape{
Assume that $\ato = \lb p, q, r \rb$.
Then, $\dnf(\neg p) = (\neg p \wedge q \wedge r ) \vee
(\neg p \wedge q \wedge \neg r) \vee
(\neg p \wedge \neg q \wedge r) \vee
(\neg p\wedge \neg q  \wedge \neg r)$.\\
Thus $\neg \dnf(\neg p) = \neg \big( (\neg p \wedge q \wedge r ) \vee
(\neg p \wedge q \wedge \neg r) \vee
(\neg p \wedge \neg q \wedge r) \vee
(\neg p\wedge \neg q  \wedge \neg r) \big)$.
Next, by applying De Morgan laws and double negation, we obtain the following formula : $(p \vee \neg q \vee \neg r ) \wedge
(p \vee \neg q \vee r) \wedge
(p \vee q \vee \neg r) \wedge
(p\vee q  \vee r)$.
By spotting-removing clauses twice, we get $(p \vee \neg q ) \wedge
(p \vee q)$.
By iterating the spotting-removing procedure, we get $p$.
}
\end{example}

We are ready to show how a weighted set of formulae is normalized.

\begin{definition}
\upshape{
Let $f \subseteq \m{Lan}$.
We denote by $\fdn(f)$ the {\bf flat decomposition} of $f$, i.e.,
$\fdn(f)$ is the set of every clause appearing in $\cnf(f)$.

Next, we denote by $\dn$ the normalization method on $\wLan$ called the {\bf Weighted Decomposer}, i.e., $\fa \Gamma \subseteq \wLan$,
$$\dn(\Gamma) = \lb \la c, w \ra : \ex \alpha = \la f, v \ra \in \Gamma, c \in \fdn(f) \textrm{ and } w = v \rb.$$
}
\end{definition}

Let us illustrate our normalization method, $\dn$.

\begin{example}
\upshape{
    The CNF of $f = \neg(p \rightarrow q \vee \neg r) \in \lan$ is $\cnf(f) = p \wedge \neg q \wedge r$. 
    The decomposed normal form of $f$, is $\fdn(f) = 
    \{p, \neg q,  r\}$.
    Similarly, for $\alpha = \tuple{\neg(p \rightarrow q \vee \neg r), 0.6} \in \wLan$, its normalization is given by $\dn(\alpha) = \{ \tuple{p, 0.6}, \tuple{\neg q,0.6}, \tuple{r, 0.6} \}$.
    }
\end{example}

For the rest of the paper, whenever we work with a normalization method $N$ on a weighted language $\cW$, the typical instance we have in mind is $\dn$ on $\wLan$.

\section{Weighted Structured Argumentation}
\label{sec:WLA}

An argument can be seen as a pair consisting of a set of premises and a claim supported by them.
Some constraints on the premises and claim are usually considered \cite{BH01}. 
The goal of this section is to extend the notion of argument to a weighted logic.


\begin{definition} \label{def:argument} 
\upshape{
Let $\bL = \tuple{\cW, {\nci},t}$ be a weighted logic.\linebreak
    A \textbf{weighted argument} on $\bL$ is a pair $A = \tuple{\Gamma,\alpha}$ such that  
        $\Gamma$ is a finite subset of $\cW$ and $\alpha \in \cW$, 
        $\Gamma$ is consistent, 
        $\Gamma ~{\nci}~ \alpha$, 
        $\forall~\Gamma' \subset \Gamma$, $\Gamma' {\not\nci} \alpha$. 
    Let $\Arg_\bL$ be the set of all weighted arguments on $\bL$. 
    We omit subscripts like {\scriptsize $\bL$} whenever they are clear from the context.
}

\end{definition}


\begin{figure*}[t!]
    \includegraphics[width = \textwidth]{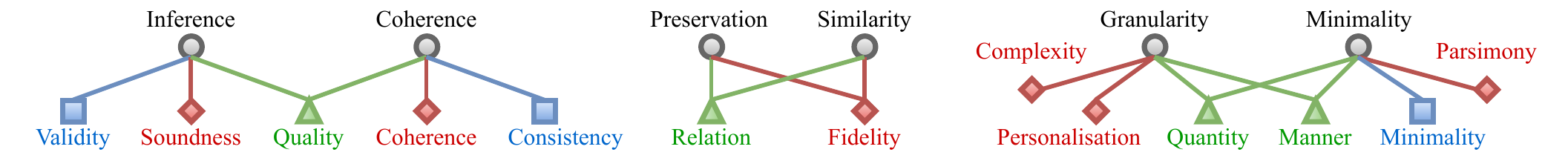}
    \caption{
    Criteria from argumentation 
    ({\color{lightblue}$\Box$}), XAI (
    {\color{red}{\LARGE$\diamond$}}), philosophy (
    {\color{darkgreen}$\bigtriangleup$}) which have inspired our decoding criteria (
    \textbf{{\footnotesize$\bigcirc$}}).}
    \label{fig:axioms}
    \vspace{-0.3cm}
\end{figure*}

However, such ideal arguments, whether weighted or not, are rarely seen. In general, humans use enthymemes, i.e., incomplete arguments in which part of the premises is missing, to logically infer the claim. The task of handling enthymemes is investigated in e.g.  \cite{hunter2007real,hunter2022understanding}.

In what follows, we introduce the notion of an approximate weighted argument, which is subject to no constraints other than the structure of its premises/claims. 
Thus, an enthymeme is a special case of this type of argument, where it is guaranteed that the inference between the premises and the claim does not logically hold.



\begin{definition}
\upshape{
Let $\bL = \tuple{\cW,{\nci},t}$ be a weighted logic.\linebreak
An {\bf approximate weighted argument} on $\bL$ is a pair $A = \tuple{\Gamma, \alpha}$ such that $\Gamma$ is a finite subset of $\cW$ and $\alpha \in \cW$. 
We denote by $\aArg_\bL$ the set of all approximate weighted arguments on $\bL$. 
An {\bf enthymeme} on $\bL$ is an element $E = \tuple{\Gamma, \alpha} \in \aArg_\bL$ such that $\Gamma {\not\nci} \alpha$.
We denote by $\ent_\bL$ the set of all enthymemes on $\bL$.
}
\end{definition}


Let us formalise and extend the running example from the introduction.

\begin{example}\label{ex:run}
\upshape{
Assuming that: $h$ = Bob is happy, $w$ = Bob is wealthy, $r$ = Bob is a researcher, $p$ = Bob gives love to people, $l$ = Bob receives love.  
Then, 

\begin{itemize}
    \item $E = \tuple{\{\tuple{w,0.7}, \tuple{r,0.7}, \tuple{p,0.8}, \tuple{l,0.9}\}, \tuple{h,0.7}}$;
    \item $D_1 = \tuple{\{\tuple{r,0.7}, \tuple{\neg r \vee h, 0.8}\},\tuple{h,0.7}}$;
    \item $D_2 = \tuple{\{\tuple{p,0.8}, \tuple{l,0.9}, \tuple{\neg p \vee \neg l \vee h, 0.9}\},\tuple{h,0.7}}$; 
    \item $D_3 = \tuple{\{\tuple{\neg r,0.7}, \tuple{w,0.7}, \tuple{\neg w \vee h,0.8} \}, \tuple{h,0.7}}$.
\end{itemize}
Where $E, D_2 \in \ent$ are enthymemes, while $D_1 \in \Arg$ is a weighted argument, and $D_3 \in \aArg$ is just an approximate weighted argument (i.e., $D_3$ is not an enthymeme). Moreover, $E,D_1,D_2,$ and $D_3$ are all normalized by $\dn$.

}
\end{example}

We are now ready to formally introduce the notion of enthymeme decoding, which, given an enthymeme and an approximate weighted argument (a decoding), returns how well it explains the potential argument underlying the enthymeme.
Note that we define a decoding without any constraints, which is justified by the fact that in real cases, we may need to evaluate imperfect decodings. In particular, if the decodings are proposed by humans or if we are automatically searching for additional information to explain the implicit, this information may be approximately coherent (e.g., in decoding $D_2$, the weight of the inference from the premises is not exactly aligned with the weight of the claim, with a difference of 0.1). We aim to evaluate any possible decoding; our evaluation criteria are specifically there to quantify the quality of the decoding.

\begin{definition}
\upshape{
$\bL = \tuple{\cW,{\nci},t}$ be a weighted logic. 
An {\bf enthymeme decoding} on $\bL$ is a pair $\la E, D \ra \in \ent \times \aArg$. Intuitively, $D$ is a decoding of the enthymeme $E$. 
}
\end{definition}

\section{Criterion Measures and Axioms}

Obviously, certain enthymeme decodings are not reasonable. 
By reasonable, we mean that there are a range possible features we would expect to see satisfied in an acceptable enthymeme decoding. 
In order to distinguish between the reasonable ones and the others, we introduce seven criteria, as well as the notion of criterion measure.



\begin{definition}
    \upshape{
    Let $\bL = \tuple{\cW,{\nci},t}$ be a weighted logic.\linebreak
    A {\bf criterion measure} on $\bL$ is a { measure of the success} of an enthymeme decoding with regard to one criterion, i.e., it is a function $\bM: \ent \times \aArg \rightarrow [0, 1]$. 
    
    
    }

\end{definition}

We propose 7 criteria for evaluating enthymeme decodings:  
the \textit{inference} of the claim from the premises, the \textit{coherence} of the premises, their \textit{minimality}, the \textit{preservation} of the enthymeme premises, the \textit{similarity} between the enthymeme premises and the decoded ones, the \textit{granularity} of the decoded premises, and the \textit{stability} of the weights.
All these criteria except stability (which is specific to our framework), are inspired by criteria defined in argumentation \cite{SL92}, or informally discussed in explainable AI (XAI) \cite{sokol2020explainability} or in philosophy \cite{grice1975logic}, as elucidated in Figure~\ref{fig:axioms}.
It is useful also to recall that the notions of argument and explanation are close \cite{hahn2023argument}, and that XAI's informal properties are originally based on social science research, to make algorithmic explanations more natural for users; which in the case of enthymeme decoding (context- and user-dependent), 
is very relevant.


For each criterion $Z$, we establish one or several axioms that a measure centered on $Z$ should satisfy.

\vspace{0.2cm}

{\it Axioms about the inference criterion.} They force a measure to consider a decoding as reasonable if the decoded premises infers the claim (Ideal version), or the more the premises fully infer the claim, the better the decoding (Increasing version). 



\begin{definition}
\upshape{
We denote by $\card{X}$ the cardinality of $X$.

Let $\bL = \tuple{\cW,{\nci},t}$ be a weighted logic and $\bM$ a criterion measure on $\bL$.
We say that $\bM$ satisfies the axioms \textbf{Ideal Flat Inference}, and \textbf{Ideal Weighted Inference} iff $\fa E \in \ent$, $\fa D= \tuple{\Delta,\beta} \in \aArg$, the following first, and second point holds, respectively:

\begin{itemize}
    \item if $\Flat(\Delta) \vdash \Flat(\beta)$, then $\bM(E,D) = 1$;
    \vspace{0.4em} 
    \item if $\Delta ~{\nci}~ \beta$, then $\bM(E,D) = 1$.
\end{itemize}

We say that $\bM$ satisfies the axiom \textbf{Lenient Increasing Flat Inference} iff, $\fa E \in \ent$, $\fa D = \tuple{\Delta,\beta}, D' = \tuple{\Delta',\beta} \in \aArg$, the following holds:
\begin{flalign*}
    & \text{if } \card{\{ f : \Flat(\Delta) \vdash f \text{ and } \Flat(\beta) \vdash f \}} \geq & \\
    & \card{\{ f : \Flat(\Delta') \vdash f \text{ and } \Flat(\beta) \vdash f \}}, & \\
    & \text{then } \bM(E,D) \geq \bM(E,D'). &
\end{flalign*}

The axiom \textbf{Strict Increasing Flat Inference} is defined as above, but $\geq$ is replaced by $>$.

We say that $\bM$ satisfies the axioms \textbf{Lenient Increasing Weighted Inference} iff, $\fa E \in \ent$, $\fa D = \tuple{\Delta,\beta}, D' = \tuple{\Delta',\beta} \in \aArg$, the following holds:
\begin{flalign*}
    & \text{if } \card{\{ \alpha : \Delta ~{\nci}~ \alpha \text{ and } \beta ~{\nci}~ \alpha \}} \geq  \card{\{ \alpha : \Delta' ~{\nci}~ \alpha \text{ and } \beta ~{\nci}~ \alpha \}}, & \\
    & \text{then } \bM(E,D) \geq \bM(E,D'). &
\end{flalign*}

The axiom \textbf{Strict Increasing Weighted Inference} is defined as above, but $\geq$ is replaced by $>$.
 
}
\end{definition}

{\it Axioms of minimality.} Decoding must be sufficiently selective to avoid overwhelming the user with data (Ideal version); the more information the premises contain that is not necessary to infer the claim, the worse the decoding (Decreasing version).
Note that if the premises do not imply the claim, then any information is potentially required to infer the claim, thus minimality is not weakened.

\begin{definition}\label{def:axMin}
\upshape{
Let $\bL = \tuple{\cW,{\nci},t}$ be a weighted logic and $\bM$ a criterion measure on $\bL$. 
We say that $\bM$ satisfies the axioms \textbf{Ideal Flat Minimality}, and \textbf{Ideal Weighted Minimality} iff 
$\forall~E = \tuple{\Gamma,\alpha} \in \ent, \fa D= \tuple{\Delta,\beta} \in \aArg$, the following first, and second point holds, respectively: 
\begin{itemize}
    \item if $\fa \Delta' \subset \Delta, \Flat(\Delta') \not{\vdash} ~ \Flat(\beta)$, then $\bM(E,D) = 1$;
    \vspace{0.25mm} 
    \vspace{-0.2cm}
    \item if $\fa \Delta' \subset \Delta, \Delta' {\not\nci} ~ \beta$, then $\bM(E,D) = 1$.
\end{itemize}


We say that $\bM$ satisfies the axioms \textbf{Lenient Decreasing Flat Minimality}, and \textbf{Lenient Decreasing Weighted Minimality} iff, $\fa E \in \ent$, $\fa D = \tuple{\Delta,\beta}, D' = \tuple{\Delta',\beta} \in \aArg$, the following first, and second point holds, respectively:
\begin{itemize}
    \item if $\card{\{\Gamma : \Gamma \subset \Delta \text{ s.t. } \Flat(\Gamma) ~{\vdash}~  \Flat(\beta)\}} \geq $ 
    \vspace{0.3em} \\
    \vspace{0.3em} 
    $\card{\{\Gamma : \Gamma' \subset \Delta' \text{ s.t. } \Flat(\Gamma) ~{\vdash}~  \Flat(\beta) \}}$, \\
    \vspace{0.5em} 
    then $\bM(E,D) \leq \bM(E,D')$;
    \vspace{0.1em} 
    
    \item if $\card{\{\Gamma : \Gamma \subset \Delta \text{ s.t. } \Gamma ~{\nci}~  \beta\}} 
    \geq$ 
    \vspace{0.3em} \\
    \vspace{0.3em} 
    $\card{\{\Gamma : \Gamma' \subset \Delta' \text{ s.t. } \Gamma ~{\nci}~  \beta \}}$, 
    \\
    \vspace{0.5em} 
    then $\bM(E,D) \leq \bM(E,D')$.
\end{itemize}

    
     
The axiom \textbf{Strict Decreasing Flat Minimality} (resp. \textbf{Strict Decreasing Weighted Minimality}) is defined as the first (resp. second) point above, but $\geq$ is replaced by $>$ and $\leq$ is replaced by $<$.

}
\end{definition}

{\it Axioms of coherence.} Any explainable system (i.e. decoding) must be consistent with itself (Strong version) or, to go further; 
any decoding must be consistent with the user's prior knowledge (Weak version).
Furthermore, the more subsets of inconsistent formulae a decoding contains, the worse the decoding (Decreasing version).

\begin{definition}
\upshape{
Let $\bL = \tuple{\cW,{\nci},t}$ be a weighted logic and $\bM$ a criterion measure on $\bL$. 
We say that $\bM$ satisfies the axioms \textbf{Ideal Strong Coherence}, and \textbf{Ideal Weak Coherence} iff, 
$\fa E  = \tuple{\Gamma,\alpha} \in \ent$, $\fa D = \tuple{\Delta,\beta} \in \aArg$, the following first, and second point holds, respectively:

\begin{itemize}
    \item $\text{if } \Delta \text{ is consistent, then } \bM(E,D) = 1$; \vspace{0.4em} 
    \item $\text{if } \Delta \cup \Gamma \text{ is consistent, then } \bM(E,D) = 1$.
\end{itemize}




We say that $\bM$ satisfies the axioms \textbf{Lenient Decreasing Strong Coherence},  iff $\fa E  = \tuple{\Gamma,\alpha} \in \ent$, $\fa D = \tuple{\Delta,\beta}, D' = \tuple{\Delta',\beta} \in \aArg$, the following holds:
\begin{flalign*}
    & \text{if } {\mid\{} \Phi \subseteq \Delta : \Phi \in \Inc \text{ and } \nexists \Psi \subset \Phi \text{ s.t. } \Psi \in \Inc {\}\mid} \geq & \\
    & {\mid\{} \Phi' \subseteq \Delta ' : \Phi' \in \Inc  \text{ and } \nexists \Psi' \subset \Phi' \text{ s.t. } \Psi' \in \Inc  {\}\mid} & \\
    & \text{then } \bM(E,D) \leq \bM(E,D'). &
\end{flalign*}


\noindent The axiom \textbf{Strict Decreasing Strong Coherence} is defined as above, but $\geq$ is replaced by $>$ and $\leq$ is replaced by $<$.

\noindent The axiom \textbf{Lenient Decreasing Weak Coherence} is defined by replacing $\Delta$ with $\Delta \cup \Gamma$ and $\Delta'$ with $\Delta' \cup \Gamma$.

\noindent The axiom \textbf{Strict Decreasing Weak Coherence} is defined by replacing $\Delta$ with $\Delta \cup \Gamma$, $\Delta'$ with $\Delta' \cup \Gamma$,  $\geq$ with $>$ and $\leq$ with $<$.





}
\end{definition}

The condition of the weak coherence is more restrictive because even if information in the premises of the enthymeme is not used in the decoding, it can prevent a decoding if the latter is inconsistent with it. 
Consequently, consistent decodings may be disallowed.
However, from a user point of view, this constraint can be very interesting. 








\begin{proposition}\label{prop:arg}
\upshape{
    Let $\bL = (\cW,{\nci},t)$ be a weighted logic,  
    $\bM,\bM',\bM''$ be 3 criterion measures on $\bL$ satisfying ideal weighted inference, any ideal coherence, and ideal minimality, respectively.
    Let $E \in \ent$ and $D \in \aArg$. If $D \in \Arg$, then
    $\bM(E,D)= \bM'(E,D) = \bM''(E,D) = 1$. 
    }
\end{proposition}



\vspace{0.1cm}

{\it Axioms of preservation.} A decoding must be based on the elements present in the enthymeme, aligned with its premises and claim.

\begin{definition}
\upshape{
Let $\bL = \tuple{\cW,{\nci},t}$ be a weighted logic, $N$ a normalization method on $\cW$, and $\bM$ a criterion measure on $\bL$. 
We say that $\bM$ satisfies the axioms \textbf{Premises $N$-Preservation}, and \textbf{Claim $N$-Preservation} iff, 
$\fa E = \tuple{\Gamma,\alpha} \in \ent$, $\fa D= \tuple{\Delta,\beta} \in \aArg$, the following first, and second point holds, respectively:
\begin{itemize}
    \item $\text{if } N(\Delta) \cap N(\Gamma) = \emptyset, \text{ then } \bM(E,D) = 0;$
    \vspace{0.4em} 
    \item $\text{if } N(\alpha) \neq N(\beta), \text{ then } \bM(E,D) = 0.$
\end{itemize}



}
\end{definition}



{\it Axioms of similarity.} Adjusting an explanation to users requires the explainability technique to model their background knowledge as much as possible, i.e. a decoding is preferable when it uses as much information as possible from the enthymeme (increasing similarity) and a minimum of new information (decreasing similarity).



\begin{definition}
\upshape{
Let $\bL = \tuple{\cW,{\nci},t}$ be a weighted logic,
$N$ a normalization method on $\cW$, and $\bM$ a criterion measure on $\bL$.
We say that $\bM$ satisfies the axiom \textbf{Lenient Increasing $N$-Similarity} iff, 
$\fa E = \tuple{\Gamma,\alpha} \in \ent,
\fa D = \tuple{\Delta,\beta}, D'= \tuple{\Delta',\beta'} \in \aArg$, 
\begin{flalign*}
&\text{if } a \geq a', b = b', c = c', 
\text{ then } \bM(E,D) \geq \bM(E,D'), &
\end{flalign*}
where $a = \card{N(\Delta) \cap N(\Gamma)}$, 
$a' = \card{N(\Delta') \cap N(\Gamma)}$, \linebreak
\phantom{where} $b = \card{N(\Delta) \setminus N(\Gamma)}$, 
$b' = \card{N(\Delta') \setminus N(\Gamma)}$, \linebreak
\phantom{where} $c = \card{N(\Gamma) \setminus N(\Delta)}$, 
$c' = \card{N(\Gamma) \setminus N(\Delta')}$. \linebreak 

\vspace{-0.2cm}

Similarly, $\bM$ satisfies the axioms \textbf{Strict Increasing $N$-Similarity}, \textbf{Lenient Decreasing $N$-Similarity}, and \textbf{Strict Decreasing $N$-Similarity} iff the following first, second, and third point holds, respectively:
\begin{itemize}
\item
$\text{if } a > a', b = b', c = c',\text{ then } \bM(E,D) > \bM(E,D');$ 
\vspace{0.2em}

\item $\text{if } a = a', b \geq b', c \geq c',\text{ then } \bM(E,D) \leq \bM(E,D');$ 
\vspace{0.2em}

\item $\text{if } a = a'$, and $(b > b', c \geq c')$ or $(b \geq b', c > c')$, \vspace{0.2em} \\
$\text{then } \bM(E,D) < \bM(E,D').$ 
\end{itemize}
}

\end{definition}

\vspace{0.1cm}

{\it Axioms of granularity.} Given the great diversity of users' experience and knowledge, a single explanation cannot meet all their expectations. 
This means that users should be able to personalize the explanation they receive according to their needs.
For example, it must respect the user's preferences regarding the granularity of an explanation, i.e. decoding.
We therefore propose two opposing strategies, aiming to prefer either concise or highly detailed decoding.
Note that, here we want to evaluate the granularity of the explanation of the implicit, and not the granularity of the argument itself. So these axioms focus only on the new formulae added in decoding and not the total set of formulae present.

\begin{definition}
\upshape{
Let $\bL = \tuple{\cW,{\nci},t}$ be a weighted logic, $N$ a normalization method on $\cW$, and $\bM$ a criterion measure on $\bL$. 
We say that $\bM$ satisfies the axiom \textbf{Lenient Concise $N$-Granularity} iff, 
$\fa E = \tuple{\Gamma,\alpha} \in \ent,
\fa D = \tuple{\Delta,\beta}, D'= \tuple{\Delta',\beta'} \in \aArg$, 
\begin{flalign*}
&\text{if } a \leq b, \text{ then } \bM(E,D) \geq \bM(E,D'), &
\end{flalign*}
where $a = \card{N(\Delta) \setminus N(\Gamma)}$ and
$b = \card{N(\Delta') \setminus N(\Gamma)}$.
\vspace{0.2em}

Similarly, $\bM$ satisfies the axioms \textbf{Strict Concise $N$-Granularity}, \textbf{Lenient Detailed $N$-Granularity}, and \textbf{Strict Detailed $N$-Granularity} iff the following first, second, and third point holds, respectively:
\begin{itemize}
    \item $\text{if } a < b, \text{then } \bM(E,D) > \bM(E,D')$; \vspace{0.2em}
    \item $\text{if } a \geq b, \text{then } \bM(E,D) \geq \bM(E,D')$; \vspace{0.2em}
    \item $\text{if } a > b, \text{then } \bM(E,D) > \bM(E,D')$.
\end{itemize}


}
\end{definition}


\vspace{0.1cm}

{\it Axioms of stability.} Finally, the aim of the last axioms is to validate the acceptable difference of weight between the initial argument (i.e., enthymeme) and its decoding. 
In the best case, the difference is zero (Ideal version), otherwise the more the difference increases, the worse the decoding (Decreasing version).


\begin{definition}
\upshape{
Let $\cW$ be a weighted language.
A {\bf weight aggregator} on $\cW$ is a function producing a  weight for a set of weighted formulae, i.e., it is a function $V : 2^{\cW} \rightarrow [0,1]$.
}
\end{definition}

\begin{definition}
\upshape{
We denote by $\abso{x}$ the absolute value of $x$. 
Let $\bL = \tuple{\cW,{\nci},t}$ be a weighted logic, $V$ a weight aggregator on $\cW$, and $\bM$ a criterion measure on $\bL$. 
We say that $\bM$ satisfies the axiom \textbf{Ideal $V$-Stability} iff, 
$\fa E \in \ent, \fa D= \tuple{\Delta,\beta}, D'= \tuple{\Delta',\beta'} \in \aArg$, the following holds: 
\begin{flalign*}
&\text{if } V(\Delta) = V(\beta), \text{ then } \bM(E,D) = 1.&
\end{flalign*}

Similarly, $\bM$ satisfies the axioms \textbf{Lenient Decreasing $V$-Stability} iff the following holds: 
\begin{flalign*}
&\text{if } \abso{V(\Delta) - V(\beta)} \geq  \abso{V(\Delta') - V(\beta')},&\\
&\text{then } \bM(E,D) \leq \bM(E,D').&
\end{flalign*}



The axiom \textbf{Strict Decreasing $V$-Stability} is defined as above, but $\geq$ is replaced by $>$ and $\leq$ is replaced by $<$ .


}
\end{definition}

{\it Relations between axioms.}
A set of axioms is {\it inconsistent} if no single criterion measure satisfies all its elements. Otherwise, it is {\it consistent}.
For example, the collection of the axioms Strict Concise Granularity and Strict Detailed Granularity is inconsistent. Most pairs of axioms presented in this section is consistent.

An axiom implies another if, for all measures, the satisfaction of the first axiom entails the satisfaction of the second one.
For instance, Weak Coherence is implied by Strong Coherence. In addition, any lenient version of an axiom is implied by its strict version (i.e., increasing/decreasing similarity, concise/detailed granularity, decreasing stability).

\section{Construction of Criterion Measures}\label{sec:4.2}














In the present section, we will construct criterion measures for each of the seven aforementionned criterion.

\vspace{0.2em}

{\it Criterion measures of coherence.} 
We assume here the strong condition that no inconsistency is acceptable in a good decoding. Moreover, the binary nature of our measures is in line with the binary nature of the consistency threshold of a weighted logic (Definition \ref{def:wlogic}).

\begin{definition}
\upshape{
Let $\bL = \tuple{\cW,{\nci},t}$ be a weighted logic, 
$\forall~E = \la \Gamma, \alpha \ra \in \ent$, $\fa D = \la \Delta, \beta \ra \in \aArg$, we define 
\begin{flalign*}
&\mathtt{Nb\_SInc}(E,D) = &\\
&
\card{ \{\Phi \subseteq \Delta : \Phi \in \Inc , \nexists \Psi \subset \Phi \text{ s.t. }\Psi \in \Inc \} } , \text{ and } &\\
&\mathtt{Nb\_WInc}(E,D) = &\\
&\card{ \{ \Phi \subseteq \Delta \cup \Gamma : \Phi \in \Inc , \nexists \Psi \subset \Phi \text{ s.t. }\Psi \in \Inc \} }.& 
\end{flalign*}

We denote first by ${\mathtt{M}_{\bL}^\mathtt{dsc}}$ the criterion measure on $\bL$ called {\bf Divided Strong Coherence}, and second by ${\mathtt{M}_{\bL}^\mathtt{dwc}}$ the criterion measure on $\bL$ called {\bf Divided Weak Coherence}: 

\begin{center}
    ${\mathtt{M}_{\bL}^\mathtt{dsc}}(E,D) = \dfrac{1}{1 + \mathtt{Nb\_SInc}(E,D)}$
\end{center}
\begin{center}
    ${\mathtt{M}_{\bL}^\mathtt{dwc}}(E,D) = \dfrac{1}{1 + \mathtt{Nb\_WInc}(E,D)}$
\end{center}

Similarly, let $p \in (0,1]$ be a {\it penalty score}, we denote first by ${\mathtt{M}_{\bL p}^\mathtt{psc}}$ the criterion measure on $\bL$ called {\bf $p$-Penalty Strong Coherence}, and second by ${\mathtt{M}_{\bL p}^\mathtt{pwc}}$ the criterion measure on $\bL$ called \textbf{$p$-Penalty Weak Coherence}:

\vspace{0.1em}

\begin{center}
    ${\mathtt{M}_{\bL p}^\mathtt{psc}}(E,D) = \mymax \big(0, 1 - p \times \mathtt{Nb\_SInc}(E,D) \big)$
\end{center}
\begin{center}
    ${\mathtt{M}_{\bL p}^\mathtt{pwc}}(E,D) = \mymax \big(0, 1 - p \times \mathtt{Nb\_WInc}(E,D) \big)$
\end{center}
}
\end{definition}

Let us illustrate the criterion measure. 

\begin{example} (Cont. running ex.) 
\upshape{
    Let $\bL = \wLog$. We have:
    \begin{itemize}
        \item $\psco{1}(E,D_1)$ {\footnotesize =} $\pwco{1}(E,D_1)$ {\footnotesize =} $1$, and \\
        $\dsco(E,D_1)$ {\footnotesize =} 
        $\dwco(E,D_1)$ {\footnotesize =} $1$;
        \vspace{0.3em}

        \item $\psco{1}(E,D_2)$ {\footnotesize =} $\pwco{1}(E,D_2)$ {\footnotesize =} $1$, and \\
        $\dsco(E,D_2)$ {\footnotesize =} 
        $\dwco(E,D_2)$ {\footnotesize =} $1$;
        \vspace{0.3em}
              
        \item $\psco{1}(E,D_3)$ {\footnotesize =} $1$, and $\pwco{1}(E,D_3)$ {\footnotesize =} $0$ while\\
        $\dsco(E,D_3)$ {\footnotesize =} $1$, and $\dwco(E,D_3)$ {\footnotesize =} $\frac{1}{2}$.
    \end{itemize}
}
\end{example}

We turn to the axiomatic analysis of our criterion measures $\psco{}$ and $\pwco{}$.

\begin{proposition}\label{prop:sat-coh}
\upshape{
Let $\bL$ be a weighted logic. 
For any $p \in (0,1]$, $\psco{}$ satisfies the axioms Ideal Strong and Weak Coherence, as Lenient Decreasing Strong and Weak Coherence. 
For any $p \in (0,1]$, $\pwco{}$ satisfies the axioms Ideal Weak Coherence and Lenient Decreasing Weak Coherence. 
$\dsco{}$ satisfies all the axioms of Coherence.
$\dwco{}$ satisfies the axioms Ideal Weak Coherence as Lenient and Strict Decreasing Weak Coherence. 
}
\end{proposition}

\vspace{0.2em}

\textit{Criterion measures of inference.} 
To evaluate the inference criterion, we propose two parametric measures based on a threshold defining the acceptable error in relation to the weight. 
We assume here that for any weighted logic, its weighted consequence operator can be defined as a combination of a flat consequence operator (such that the flat support infers the flat claim), and a weight aggregator (such that the aggregated weight of the support equals the claim's weight).

Given that inference strongly depends on language and its consequence operator, we will propose measures specific to propositional weighted logic, in order to give a concrete example.
To reason finitely on a set of formulae, we borrow and modify from Definition 41 in \cite{david2021dealing} the definition of dependent finite Cn.
Note that even if the measures for inference proposed here are specific to this (propositional) logic, it is nevertheless possible to generalise these measures to any logic by adapting the finite inference function (here flat finite Cn).

\begin{definition} \label{cndf}
\upshape{
Let $\Delta \subseteq \wLan$, $N$ a normalization method on $\wLan$, the \textbf{flat finite Cn} is defined by $\fCN{N}(\Delta) =$ 
\begin{flalign*}
&\{ f : \Flat(\Delta) \vdash f \text{ s.t. } f \in \Flat(N(\wLan))& \\
&\text{and }  \lit(f) \subseteq \lit(\Flat(\Gamma)) \text{ where } \Gamma \subseteq \Delta  \text{ s.t.}& \\
&\Flat(\Gamma) \vdash f \text{ and } \nexists \Gamma' \subset \Gamma \text{ s.t. } 
\Flat(\Gamma') \vdash f \}.&
\end{flalign*}
}
\end{definition}

\begin{example}
\upshape{
Let $N = \dn$, and 
\begin{itemize}
    \item $\Delta = \{\tuple{r,0.7}, \tuple{\neg r \vee h, 0.8}\} \subseteq \wLan$;
    \item $\alpha = \tuple{h,0.7} \in \wLan$;
    
    \item $\beta = \tuple{r \wedge h \wedge x, 0.7} \in \wLan$.
    
\end{itemize}
Hence, we have:
\begin{itemize}
    \item $\fCN{}(\Delta) = \{r, \neg r \vee h, h, r \vee h\}$;
    \item $\fCN{}(\alpha) = \{h\}$;
    \item $\fCN{}(\beta) = \{r,h,x, r \vee h, r \vee x, h \vee x, r \vee h \vee x\}$.
\end{itemize}
}
\end{example}

It is interesting to note that the use of inferences based solely on the literals present initially avoids the explosion of clauses inferable from all possible literals (and which are not relevant here), however we have a variation of clauses for all acceptable combinations of literals; e.g., with $r$ and $h$ we will also have $r \vee h$.
This combination can be seen as a redundancy. One option would be to use implicate primes, which has been studied in the literature for compilation problems \cite{darwiche2002knowledge}, however if we compare the implicate primes of $\{r,h\}$ with those of $\{r \vee h\}$, we see no overlap although there is an inference relationship between these two set of formulae. 
For this reason we have defined the finite flat Cn operator, and we consider that semantic overlap between clause combinations is the price to pay for a fine-grained and comparable semantic representation.  

Moreover, to check for common semantic information between the premises and the claim, we also considered using models.
Unfortunately, if the premises are inconsistent, the models do not allow for detecting common inferences. For example, between the premises $\{r, \neg r, h\}$ and the claim $\{h\}$, there is no common interpretation.

Next, we present two families of measures for calculating how well the premises of a decoding infers its claim.

\begin{definition}
\upshape{ 



Let $\bL = \tuple{\cW,{\nci},t}$ be a weighted logic, $N$ a normalization method on $\cW$, $a \in [0,1]$ be an \textit{acceptable error}, and $V$ be the weight aggregator used in $\bL$. 
We denote by ${\mathtt{M}_{\bL N a V}^\mathtt{dpi}}$ the criterion measure on $\bL$ called \textbf{Divided Parametric $NaV$-Inference}, i.e.,
$\forall~E \in \ent, \fa D = \tuple{\Delta,\beta} \in \aArg$, the following holds:
\begin{flalign*}
&\text{if } \abso{V(\Delta) - V(\beta)} \leq \m{a}, &
\\
\vspace{0.7em}
&\text{then } {\mathtt{M}_{\bL N a V}^\mathtt{dpi}}(E,D) = \dfrac{\card{\fCN{N}(\beta)}}
{\card{\fCN{N}(\beta)} + \card{\fCN{N}(\beta) \setminus \fCN{N}(\Delta)}}; & 
\vspace{0.6em}
\\
\vspace{0.6em}
&\text{otherwise }  {\mathtt{M}_{\bL N a V}^\mathtt{dpi}}(E,D) = 0.&
\end{flalign*}


Similarly, let $p \in (0,1]$ be a {\it penalty score}, we denote by ${\mathtt{M}_{\bL p N a V}^\mathtt{ppi}}$ the criterion measure on $\bL$ called {\bf $p$-Penalty  Parametric $NaV$-Inference}, i.e.,
$\forall~E \in \ent, \fa D = \tuple{\Delta,\beta} \in \aArg$, the following holds:
\begin{flalign*}
& \text{if } \abso{V(\Delta) - V(\beta)} \leq \m{a},& 
\\
& \text{then } {\mathtt{M}_{\bL p N a V}^\mathtt{ppi}}(E,D) =
\mymax \big(0 , {1 - p \times }
\card{\fCN{N}(\beta) \setminus \fCN{N}(\Delta)} \big) & 
\\  
& \text{otherwise }  {\mathtt{M}_{\bL p N a V}^\mathtt{ppi}}(E,D) = 0.&
\end{flalign*}

}
\end{definition}


In our running example we do not illustrate the case where premises partially infers its claim, we extend the example here with another decoding to illustrate the different behavior of the measures.

\begin{example} (Cont. running ex.) 
\upshape{
    Let $\bL = \wLog$, $N = \dn$, $V$ be the $\mymin$ function on the weight of the formulae, and $p = 0.1$. We have:
    \begin{itemize}
        \item $\dpi{0}(E,D_1)$ {\footnotesize =} $\dpi{1}(E,D_1)$ {\footnotesize =} $1$, and\\
        $\ppi{0}(E,D_1)$ {\footnotesize =} $\ppi{1}(E,D_1)$ {\footnotesize =} $1$;
        \vspace{0.3em}
        
        \item $\dpi{0}(E,D_2) = 0$, $\dpi{1}(E,D_2) = 1$, and \\ 
        $\ppi{0}(E,D_2) = 0$, $\ppi{1}(E,D_2) = 1$;
        \vspace{0.3em}
        
        \item $\dpi{0}(E,D_3)$ {\footnotesize =} $\dpi{1}(E,D_3)$ {\footnotesize =} $1$, and \\
        $\ppi{0}(E,D_3)$ {\footnotesize =} $\ppi{1}(E,D_3)$ {\footnotesize =} $1$;
        \vspace{0.3em}

        \item let $D_4 = \tuple{\{\tuple{r,0.7}, \tuple{\neg r \vee h, 0.8}\}, \tuple{r \wedge h \wedge x, 0.7}}$:\\
        $\dpi{0}(E,D_4) = \dpi{1}(E,D_4) = \frac{7}{11} \approx 0.64$, and \\ 
        $\ppi{0}(E,D_4) = \ppi{1}(E,D_4) = 0.6$;
    \end{itemize}
} 
\end{example}


Depending on the acceptable error parameter, the criterion measures can follow more weighted inferences axioms (when $a < 1$) or flat inferences axioms (when $a = 1$).
We test $\dpi{}$, and $\ppi{}$ (for all $a$) against our axioms centred on the inference criterion, and we denote by $\dpi{1}$ and $\ppi{1}$ when $a =1$, and also by $\dpi{<1}$ and $\ppi{<1}$ for all $a \in [0,1)$.


\begin{proposition}\label{prop:sat-inf}
\upshape{
Let $\bL = \tuple{\cW,{\nci},t}$ be a weighted logic, $N$ a normalization method on $\cW$, $a \in [0,1]$ be an \textit{acceptable error}, and $V$ be the weight aggregator used in $\bL$. 
The measure $\ppi{1}$ satisfies the axioms Ideal Weighted and Flat Inference, as well as Lenient Increasing Weighted and Flat Inference.
The measures $\ppi{<1}$ satisfy the axioms Ideal Weighted Inference, and Lenient Increasing Weighted Inference. 
The measures $\dpi{<1}$ satisfies the axioms Ideal Weighted Inference, Lenient and Strict Increasing Weighted Inference. 
The measure $\dpi{1}$ satisfies all the axioms of Inference.

}
\end{proposition}

\textit{Criterion measures of minimality.} 
For the minimality criterion, we propose two strategies: one based on the number of minimal subsets, and another based on the number of unnecessary formulae.

Since we count knowledge, we apply a normalization method to it prior to counting.

\begin{definition}\label{def:cmin}
\upshape{
Let $\bL = \tuple{\cW,{\nci},t}$ be a weighted logic, and $N$ a normalization method on $\cW$. 
We denote by $\cinf_{\bL N}$ the function on $2^{\cW} \times \cW$ such that,
$\fa \Delta \subseteq \cW$, $\fa \beta \in \cW$, the following holds:
$$\cinf_{\bL N}(\Delta,\beta) = \lb \Gamma : \Gamma \subseteq N(\Delta) \textrm{ and } \Flat(\Gamma) \vdash \Flat(\beta) \rb.$$

Let $\cmmin{\bL N}$ be the criterion measure on $\bL$ called the \textbf{Divided $N$-Minimality}, i.e.,
$\forall~E \in \ent, \forall~D = \tuple{\Delta,\beta} \in \aArg$, 
%
\begin{flalign*}
&\text{if } \cinf(\Delta,\beta) = \emptyset, \text{ then } \cmmin{\bL N}(E,D) = 1;&\\
&\text{otherwise, } \cmmin{\bL N}(E,D) = \dfrac{1}{\card{\cinf(\Delta,\beta)}}.&
\end{flalign*}

In addition, let $p \in (0,1]$ be a {\it penalty score}.
We denote by $\cmpen{\bL p N}$ the criterion measure on $\bL$ called \textbf{$p$-Penalty $N$-Minimality}, i.e.,
\vspace{0.3em}
$\forall~E \in \ent, \forall~D = \tuple{\Delta,\beta} \in \aArg$,
\begin{flalign*}
&\text{if } \cinf(\Delta,\beta) = \emptyset, \text{ then } \cmpen{\bL p N}(E,D) = 1;&\\
&\text{otherwise, } \cmpen{\bL p N}(E,D) = &\\
&\mymax \big( 0, 
1 - p \times \big( |\Delta| - \mymin \lb |\Gamma| : \Gamma \in \cinf(\Delta,\beta) \rb \big) \big).&
\end{flalign*}
}
\end{definition}





We turn to our running example.

\begin{example} (Cont. running ex.) 
\upshape{ 
    Let $\bL = \wLog$, $N = \dn$, 
    and $p = \frac{1}{4}$. We have:
    \begin{itemize}
        \item $\cmmin{}(E,D_1) = 1$, and $\cmpen{}(E,D_1) = 1$;
        \vspace{0.2em}
        
        \item $\cmmin{}(E,D_2) = 1$, and $\cmpen{}(E,D_2) = 1$;
        \vspace{0.2em}
        
        \item $\cmmin{}(E,D_3) = \frac{1}{2}$, and  $\cmpen{}(E,D_3) = \frac{3}{4}$.
    \end{itemize}
}
\end{example}


We test $\cmmin{}$ and $\cmpen{}$ against our axioms.

\begin{proposition}\label{prop:sat-min}
\upshape{
Let $\bL = \tuple{\cW,{\nci},t}$ be a weighted logic, and $N$ a normalization method on $\cW$. 
$\cmmin{}$ satisfies all the axioms of Minimality.
Let $p \in (0, 1]$, $\cmpen{}$ satisfies the axioms Ideal Flat and Weighted Minimality, as Lenient Decreasing Flat and Weighted Minimality.
}
\end{proposition}

\vspace{0.2em}



\textit{Criterion measures of similarity.} 
On the following, we propose syntactic similarity measure from the literature to decode the criterion of similarity.


Tversky's ratio model \cite{Tversky77} is a general similarity measure which encompasses different well known similarity measure such as \cite{Jaccard}, 
\cite{Dice}, \cite{Sorensen}, 
\cite{Anderberg} and \cite{Sneath}.
These measures have been studied in the literature to evaluate arguments in propositional logic \cite{AmgoudD18,AmgoudDD19} and first-order logic \cite{davidsimilarity23}.

\begin{definition} \label{def:extension-tversky-measure}
\upshape{
Let $\cW$ be a weighted language, $N$ a normalization method on $\cW$,  $\Gamma, \Delta \subseteq \cW$, and $x, y \in (0,+\infty)$.
We denote by $\m{Tve}_{N}(\Gamma,\Delta, x, y)$ the \textbf{$Nxy$-Tversky Measure}, i.e., 
$$\m{Tve}_{N}(\Gamma,\Delta, x, y) = 
 			\left\{
 			\begin{array}{l l}
 			1 & \textrm{if } \Gamma = \Delta = \emptyset;\\
 			\dfrac{a}{a + x \times  b + y \times c} & \textrm{otherwise,}\\
 			\end{array}
 			\right.$$
    where $a= \card{N(\Gamma) \cap N(\Delta)}$, $b = \card{N(\Gamma) \setminus N(\Delta)}$, and \\
    \phantom{where} $c = \card{N(\Delta) \setminus N(\Gamma)}$.
}
\end{definition}

The above classic measures can be obtained with $\alpha = \beta = 2^{-n}$.
In particular, the Jaccard measure is obtained with $n = 0$ (i.e., $\m{Tve}_{1,1} = \m{jac}$), Dice with $n = 1$ (i.e., $\m{Tve}_{0.5,0.5} = \m{dic}$), Sorensen with $n = 2$ (i.e., $\m{Tve}_{0.25,0.25} = \m{sor}$), Anderberg with $n = 3$ (i.e., $\m{Tve}_{0.125,0.125} = \m{and}$), and Sokal and Sneah 2 with $n = -1$ (i.e., $\m{Tve}_{2,2} = \m{ss2}$). 


\begin{definition}
\upshape{
Let $\bL = \tuple{\cW,{\nci},t}$ be a weighted logic, $N$ a normalization method on $\cW$, and $x, y \in (0,+\infty)$.
We denote by $\cmtve{\bL N x y}$ the criterion measure on $\bL$ called the \textbf{$Nxy$-Tversky Similarity} on $x$ and $y$, i.e.,
$\forall~E = \tuple{\Gamma,\alpha} \in \ent, \fa D = \tuple{\Delta,\beta} \in \aArg$,
\vspace{0.2em}
\begin{center}
$\cmtve{\bL N x y}(E,D) = \m{Tve}(\Gamma,\Delta, x, y).$
\end{center}

}
\end{definition}

Note that, with a similarity measure, the score of 1 is obtained when the decoding is identical to the enthymeme. 
Since an enthymeme, by definition, is not correct, a good decoding should never score 1 with a similarity measure.

\begin{example} (Cont. running ex.) 
\upshape{
    Let $\bL = \wLog$, and $N = \dn$.
    We have:
    \begin{itemize}
        \item 
        $\cmtve{\m{and}}(E,D_1) = \frac{1}{1.5}$, $~~$ and $~\cmtve{\m{ss2}}(E,D_1) = \frac{1}{9}$;
        \vspace{0.2em}

        \item 
        $\cmtve{\m{and}}(E,D_2) = \frac{2}{2.375}$, and $~\cmtve{\m{ss2}}(E,D_2) = \frac{2}{8}$;
        \vspace{0.2em}
        
        \item 
        $\cmtve{\m{and}}(E,D_3) = \frac{1}{1.625}$, and $~\cmtve{\m{ss2}}(E,D_3) = \frac{1}{11}$.
    \end{itemize}
}
\end{example}

We analyze $\cmtve{}$ on the basis of our axioms.

\begin{proposition}\label{prop:sat-sim}
\upshape{
Let $\bL = \tuple{\cW,{\nci},t}$ be a weighted logic and $N$ a normalization method on $\cW$. $\cmtve{\m{jac}}$, $\cmtve{\m{dic}}$, $\cmtve{\m{sor}}$, $\cmtve{\m{and}}$, and $\cmtve{\m{ss2}}$ satisfy the axioms of lenient, strict, increasing, decreasing similarity.
 }
\end{proposition}

\vspace{0.2em}

\textit{Criterion measures of preservation.} 
We propose criterion measures which are generalizations of the ones for similarity criterion, and another one which focus only on the criterion of preservation. 

\begin{definition}\label{def:preserv}
\upshape{
Let $\bL = \tuple{\cW,{\nci},t}$ be a weighted logic, $N$ a normalization method on $\cW$, and $x, y \in (0,+\infty)$.
We denote by $\cmtvetve{\bL N x y}$ the criterion measure on $\bL$ called the \textbf{$Nxy$-Tversky Preservation} on $x$ and $y$, i.e.,
$\forall~E = \tuple{\Gamma,\alpha} \in \ent, \fa D = \tuple{\Delta,\beta} \in \aArg$,
$$\cmtvetve{\bL N x y}(E,D) = \m{Tve}(\Gamma,\Delta, x, y) \times \m{Tve}(\alpha,\beta, x, y).$$

Next, we denote by $\cmbl{\bL N}$ the criterion measure on $\bL$ called the \textbf{Basic $N$-Preservation}, i.e., $\forall~E = \tuple{\Gamma,\alpha} \in \ent, \fa D = \tuple{\Delta,\beta} \in \aArg$, the following holds:
\begin{flalign*}
 & \text{if } |N(\Gamma) \cap N(\Delta)| \times |N(\alpha) \cap N(\beta)| > 0,&\\
 &\text{then } \cmbl{\bL N}(E,D) = 1; &\\
 &\text{otherwise, } \cmbl{\bL N}(E,D) = 0.&
\end{flalign*}
}
\end{definition}

Let us illustrate the 
definition on our running example.

\begin{example} (Cont. running ex.) 
\upshape{
    Let  $\bL = \wLog$, and $N = \dn$.
    We have:
    \begin{itemize}
        \item $\cmtvetve{\m{and}}(E,D_1) = \frac{1}{1.375}$, $\cmtvetve{\m{ss2}}(E,D_1) = \frac{1}{7}$, and \\
        $\cmbl{}(E,D_1) = 1$;
        \vspace{0.3em}
        
        \item $\cmtvetve{\m{and}}(E,D_2) = ~~\frac{1}{1.5}$, $~\cmtvetve{\m{ss2}}(E,D_2) = \frac{1}{9}$, and \\
        $\cmbl{}(E,D_2) = 1$;
        \vspace{0.3em}
        
        \item $\cmtvetve{\m{and}}(E,D_3) = ~\frac{3}{3.25}$, $~\cmtvetve{\m{ss2}}(E,D_3) = \frac{3}{7}$, and \\
        $\cmbl{}(E,D_3) = 1$.
    \end{itemize}
}
\end{example}

We test $\cmtvetve{}$ and $\cmbl{}$ against our axioms.

\begin{proposition}\label{prop:sat-link}
\upshape{
Let $\bL = \tuple{\cW,{\nci},t}$ be a weighted logic, $N$ a normalization method on $\cW$. 
$\cmtvetve{\m{jac}}$, $\cmtvetve{\m{dic}}$, $\cmtvetve{\m{sor}}$, $\cmtvetve{\m{and}}$, $\cmtvetve{\m{ss2}}$, and $\cmbl{}$ satisfy the axioms of Premises and Claim Preservation.
 }
\end{proposition}

\vspace{0.2em}




\textit{Criterion measures of granularity.} 
Let us start by looking at the criterion measures of the granularity criterion with a strategy preferring concise decodings. 
Once again, we propose a version based on the division operator (which has a strict behavior) and a version with a user-defined penalty (i.e., lenient).

\begin{definition}\label{def:graCon}
\upshape{
Let $\bL = \tuple{\cW,{\nci},t}$ be a weighted logic and $N$ a normalization method on $\cW$.
We denote by $\cmcd{\bL N}$ the criterion measure on $\bL$ called the \textbf{Concise Divided $N$-Granularity}, i.e., 
$\forall~E = \tuple{\Gamma,\alpha} \in \ent, \fa D = \tuple{\Delta,\beta} \in \aArg$, the following holds:
$$\cmcd{\bL N}(E,D) = 
\frac{1}{\card{N(\Delta) \setminus N(\Gamma)} +1}.$$

Next, let $s \in \mathbb{N}^+$ (where $\mathbb{N}^+ = \mathbb{N} \setminus \{0\}$)  be a {\it maximal detail size} and
$p \in (0,1]$ a {\it penalty score}.
We denote by $\cmcp{\bL s p N}$ the criterion measure on $\bL$ called the \textbf{Concise $sp$-Penalty $N$-Granularity}, i.e., 
$\forall~E = \tuple{\Gamma,\alpha} \in \ent, D = \tuple{\Delta,\beta} \in \aArg$, the following holds:
\begin{flalign*}
&\text{if } |N(\Delta) \setminus N(\Gamma)| \leq s,
\text{then } \cmcp{\bL s p N}(E,D) = 1;& \\
&\text{otherwise, } \cmcp{\bL s p N}(E,D) =&\\
& \mymax \big( 0, 1 - p \times (\card{N(\Delta) \setminus N(\Gamma)} - s) \big).&
\end{flalign*}
}
\end{definition}


\begin{example} (Cont. running ex.) 
\upshape{ 
Let $\bL = \wLog$, $N = \dn$,
$s = 1$, and $p = 0.5$. We have:
    \begin{itemize}
        \item $\cmcd{}(E,D_1) = \frac{1}{2}$, and $\cmcp{}(E,D_1) = 1$;
        \vspace{0.2em}
        
        \item $\cmcd{}(E,D_2) = \frac{1}{2}$, and $\cmcp{}(E,D_2) = 1$;
        \vspace{0.2em}
        
        \item $\cmcd{}(E,D_3) = \frac{1}{3}$, and $\cmcp{}(E,D_3) = \frac{1}{2}$.
    \end{itemize}
}
\end{example}

\begin{table*}[h!]
\centering
\small
$\begin{array}{|l|c|c| |l|c|c| |l|c|c|c|c| }
\hline
     & \cmbl{}  & \cmtvetve{} &   & \cmsd{}  & \cmld{} &  & \ppi{1}  & \ppi{<1} & \dpi{1} & \dpi{<1} \\ 
    \hline
    
    \textrm{P. Preservation} & \bullet & \bullet & \textrm{I. Stability} & \bullet & \bullet & 
    \textrm{I.F. Inference}   & \bullet  &   & \bullet &  \\
    
    \textrm{C. Preservation} & \bullet & \bullet & \textrm{L.D. Stability}  & \bullet &  \bullet & 
    \textrm{I.W. Inference} & \bullet  & \bullet & \bullet & \bullet \\
    
      &   &   & \textrm{S.D. Stability} & \bullet & & \textrm{L.I.F. Inference}& \bullet & &\bullet &  \\ 
    
    \cline{1-6}
    
     & \cmcd{} & \cmcp{} & & \cmdg{} & \cmpg{} & \textrm{S.I.F. Inference} & & &\bullet &   \\ 
    \cline{1-6}
    
    \textrm{L.C. Granularity} & \bullet & \bullet & \textrm{L.D. Granularity} & \bullet & \bullet & \textrm{L.I.W. Inference} & \bullet & \bullet & \bullet & \bullet \\ 
    
    \textrm{S.C. Granularity} & \bullet & & \textrm{S.D. Granularity} & \bullet & & \textrm{S.I.W. Inference} & & & \bullet & \bullet \\ 
    \hline	
    \hline
    
     & \multicolumn{2}{c||}{\cmtve{}}  &  & \cmmin{} & \cmpen{} & & \psco{} & \pwco{} & \dsco & \dwco  \\
    \hline
    \textrm{L.I. Similarity} & \multicolumn{2}{c||}{\bullet} & \textrm{I.F. Minimality} & \bullet & \bullet &  \textrm{I.S. Coherence}  & \bullet &  & \bullet &  \\
    
    \textrm{S.I. Similarity} & \multicolumn{2}{c||}{\bullet}  & \textrm{I.W. Minimality} & \bullet & \bullet & \textrm{I.W. Coherence}   & \bullet & \bullet & \bullet & \bullet  \\
    
    \textrm{L.D. Similarity} & \multicolumn{2}{c||}{\bullet}  & \textrm{L.D.F. Minimality}   & \bullet  & \bullet & \textrm{L.D.S. Coherence} & \bullet & & \bullet &  \\
    
    \textrm{S.D. Similarity} & \multicolumn{2}{c||}{\bullet}  & \textrm{L.D.W. Minimality} & \bullet & \bullet & \textrm{S.D.S. Coherence} & & & \bullet &  \\

      & \multicolumn{2}{c||}{}  & \textrm{S.D.F. Minimality} & \bullet & & \textrm{L.D.W. Coherence} & \bullet & \bullet & \bullet & \bullet  \\

      &  \multicolumn{2}{c||}{}  & \textrm{S.D.W. Minimality} & \bullet &  & \textrm{S.D.W. Coherence} & & & \bullet & \bullet  \\
    \hline  
    \end{array}$
    \caption{Axioms and Measures, where $\bullet$ means that the corresponding measure satisfies the corresponding axiom.}
\label{tablefinale}
\end{table*}

We turn to the axiomatic analysis of $\cmcd{}$ and $\cmcp{}$.

\begin{proposition}\label{prop:sat-cgran}
\upshape{
Let $\bL = \tuple{\cW,{\nci},t}$ be a weighted logic and $N$ a normalization method on $\cW$. $\cmcd{}$ satisfies Lenient and Strict Concise Granularity. 
Let $s \in \mathbb{N}^+$ and $p \in (0,1]$. $\cmcp{}$ satisfies Lenient Concise Granularity.
 }
\end{proposition}




Next, we propose the dual versions of the previous criterion measures.

\begin{definition}\label{def:graDet}
\upshape{
Let $\bL = \tuple{\cW,{\nci},t}$ be a weighted logic and $N$ a normalization method on $\cW$.
Let $\cmdg{\bL N}$ be the criterion measure on $\bL$ called 
the \textbf{Detailed Divided $N$-Granularity}, i.e., 
$\forall~E = \tuple{\Gamma,\alpha} \in \ent, D = \tuple{\Delta,\beta} \in \aArg$, 
\begin{center}
$\cmdg{\bL N}(E,D) = 
1 - \dfrac{1}{\card{N(\Delta) \setminus N(\Gamma)} +1}.$
\end{center}

Next, let $s \in \mathbb{N}^+$ be a minimal detail size and
$p \in (0,1]$ a penalty score.
We denote by $\cmpg{\bL s p N}$ the criterion measure on $\bL$ called the
\textbf{Detailed $sp$-Penalty $N$-Granularity}, i.e.,
\begin{flalign*}
&\text{if } |N(\Delta) \setminus N(\Gamma)| \geq s,
\text{then } \cmpg{\bL s p N}(E,D) = 1;& \\
&\text{otherwise, } \cmpg{\bL s p N}(E,D) = &\\
&\mymax \big( 0, 1 - p \times (s - \card{N(\Delta) \setminus N(\Gamma)}) \big).&
\end{flalign*}
}
\end{definition}


\begin{example} (Cont. running ex.) 
\upshape{
    Let $\bL = \wLog$, 
    $N = \dn$, 
    $s = 1$ and $p = \frac{1}{2}$. We have:
    \begin{itemize}
        \item $\cmdg{}(E,D_1) = \frac{1}{2}$, and $\cmpg{}(E,D_1) = 1$;
        \vspace{0.2em}
        
        \item $\cmdg{}(E,D_2) = \frac{1}{2}$, and $\cmpg{}(E,D_2) = 1$;
        \vspace{0.2em}
        
        \item $\cmdg{}(E,D_3) = \frac{2}{3}$, and $\cmpg{}(E,D_3) = 1$.
    \end{itemize}
}
\end{example}

Let us analyze $\cmdg{}$ and $\cmpg{}$ with our axioms.

\begin{proposition}\label{prop:sat-dgran}
\upshape{
Let $\bL = \tuple{\cW,{\nci},t}$ be a weighted logic and $N$ a normalization method on $\cW$.
$\cmdg{}$ satisfies Lenient and Strict Detailed Granularity. 
Let $s \in \mathbb{N}^+$ and $p \in (0,1]$. $\cmpg{}$ satisfy Lenient Detailed Granularity.
 }
\end{proposition}



\vspace{0.2cm}

\textit{Criterion measures of stability.} 
We propose a strict version discriminating all variations from the difference, and a more adaptable version encompassing intervals of difference as acceptable or unacceptable according to two thresholds.

\begin{definition}\label{def:stabC}
\upshape{
Let $\bL = \tuple{\cW,{\nci},t}$ be a weighted logic.
We denote by $\cmsd{\bL}$ the criterion measure on $\bL$ called the \textbf{Strict Difference Stability}, i.e.,
$\forall~E \in \ent, \forall D = \tuple{\Delta,\beta} \in \aArg$, the following holds:
\begin{flalign*}
& \text{if } \Delta = \emptyset \text{, then }\cmsd{\bL}(E,D) = 1;&\\
& \text{otherwise, }& \\
& \cmsd{\bL}(E,D) = 1 - \abso{\mymin[\Wei(\Delta)] - \Wei(\beta)}.&
\end{flalign*}

Next, let $a \in [0,1]$ be an {\it acceptable error} (with no impact) and $u \in (0,1]$ be an {\it unacceptable error} (nullifying the evaluation) such that $a < u$.

We denote by $\cmld{\bL a u}$ the criterion measure on $\bL$ called
the \textbf{Lenient $au$-Difference Stability}, i.e., $\forall~E \in \ent, \forall D = \tuple{\Delta,\beta} \in \aArg$,  where $Err = \abso{\mymin[\Wei(\Delta)] - \Wei(\beta)}$, the following holds:
\begin{flalign*}
&\text{if } \Delta = \emptyset \text{, then } \cmld{\bL a u}(E,D) = 1; & \\
&\text{if } \Delta \not= \emptyset \text{ and } Err \leq a, \text{ then } \cmld{\bL a u}(E,D) = 1; & \\
& \text{if } \Delta \not= \emptyset \text{ and } u \leq Err, \text{ then } \cmld{\bL a u}(E,D) = 0; & \\
& \text{if } \Delta \not= \emptyset \text{ and } a < Err < u,
& \\
& \text{then }\cmld{\bL a u}(E,D) = 1 - \dfrac{Err - a}{u - a}. & 
\end{flalign*}
}
\end{definition}

For $\cmld{}$, we propose to re-scale the difference according to the acceptable error (i.e., $a$) and unacceptable error (i.e., $u$) bounds.
This can be used if the user want to increase the importance of this criterion.

\begin{example} (Cont. running ex.) 
\upshape{
    Let $\bL = \wLog$, $a = 0$, $u = \frac{3}{10}$, $u' = \frac{1}{2}$. We have:
    \begin{itemize}
        \item $\cmsd{}(E,D_1) = 1$,  $~~~\cmld{u}(E,D_1) = 1$,  
        $~\cmld{u'}(E,D_1) = 1$;
        \vspace{0.2em}
        
        \item $\cmsd{}(E,D_2) = \frac{9}{10}$,\phantom{..}$\cmld{u}(E,D_2) = \frac{2}{3}$,  $\cmld{u'}(E,D_2) = \frac{4}{5}$;
        \vspace{0.2em}
        
        \item $\cmsd{}(E,D_3) = 1$,  $~~~\cmld{u}(E,D_3) = 1$,  
        $~\cmld{u'}(E,D_3) = 1$.
    \end{itemize}
}
\end{example}

We turn to our final axiomatic analysis of measures.

\begin{proposition}\label{prop:sat-wstab}
\upshape{
Let $\bL = \tuple{\cW,{\nci},t}$ be a weighted logic.
$\cmsd{}$ satisfies Ideal Stability, as Lenient and Strict Decreasing Stability. 
Let $a, u \in [0,1]$ such that $a < u $. $\cmld{}$ satisfies Ideal Stability and Lenient Decreasing Stability.
 }
\end{proposition}

\section{Quality Measure}

Criterion measures look at different aspects of the quality of a decoding of an enthymeme.
In order to get a better understanding of the quality of a decoding, we will use multiple criterion measures, each giving a value, and then we combine those values to give a single quality measure. 

An {\bf aggregation function} is a function $\F:[0,1]^n \rightarrow [0,1]$, where $n \in \mathbb{N}$, which aggregates a sequence of values into a single one.

\begin{definition}
\upshape{
Let $\bL = \langle \cW,{\nci},t \rangle$ be a weighted logic, $\E = \tuple{\bM_1,\ldots,\bM_k}$ a sequence of criterion measures on $\bL$, and $\F$ an aggregation function.
We denote by $\qual^\E_\F$ the {\bf quality measure} based on $\E$ and $\F$, i.e., the function on $\ent \times \aArg$ such that, $\fa E \in \ent$, $\fa D \in \aArg$, the following holds:
\begin{center}
    $\qual^\E_\F(E,D) = \F\langle v_1, \ldots, v_k \rangle,$ where \\
    \vspace{0.3em}
    $\bM_1(E,D) = v_1, \ldots, \bM_k(E,D) = v_k$.
\end{center}
}
\end{definition}

Let see some specific examples of aggregation function.

\begin{definition}\label{def:agg}
\upshape{
Let a sequence $T = \langle v_1, \ldots, v_k \rangle$ where each $v_i \in [0,1]$.  The following \textbf{aggregation functions} 
$\m{F}^\m{av}$, and $\m{F}^\m{pr}$ are defined as follows:
\begin{itemize}
\item
if $\card{T} = 0$, then $\m{F}^\m{av}(T) = 0$, else
 $\m{F}^\m{av}(T) = 
 \frac{\sum_{i=1}^{\card{T}} T[i]}{\card{T}}$   
\item 
if $\card{T} = 0$, then $\m{F}^\m{pr}(T) = 0$, else
    $\m{F}^\m{pr}(T) = \prod_{i=1}^{\card{T}} T[i]$
\end{itemize}
}
\end{definition}

Let us see now two examples of set of criteria.




\begin{definition}
\upshape{
Let 
the \textbf{Lenient detailed} $\m{Ld}$ and the \textbf{Strict detailed} $\m{Sd}$, sequence of criterion measures, defined as:
\begin{itemize}
    \item 
    $\m{Ld} = \langle$
    $\psco{1},$
    $\ppi{1},$
    $\cmpen{\frac{1}{4}},$
    $\cmbl{},$
    $\cmtve{\m{and}},$
    $\cmld{0,\frac{1}{2}},$
    $\cmpg{1,\frac{1}{2}}\rangle$

\item 
    $\m{Sd} = \langle$
    $\pwco{1},$
    $\dpi{0},$
    $\cmmin,$
    $~\cmbl{},$ 
    $\cmtve{\m{ss2}},$
    $\cmsd{},$
    $~~\cmdg{}\rangle$
\end{itemize}

}
\end{definition}






Let us motivate $\m{Ld}$ with examples of practical applications: i) lenient criteria may be desirable to analyse the scope of an enthymeme, in particular in politics where the aim is to be favourably decoded by as many people as possible;
ii) detailed granularity criterion may be more useful than the concise one or the similarity criterion, e.g., in an expert context, if the goal is to understand and thus add all the precision of the reasoning.
Similar justifications can be found for $\m{Sd}$.

Let us continue with our running example, and study the best decoding (according to different criteria and aggregations) for the enthymeme $E$ that explains why Bob is happy.


    \begin{center}
    \resizebox{0.44\textwidth}{!}{ 
    $\begin{array}{lllc}
    \hline
        \quality^{\m{Ld}}_{\m{av}}(E,D_1)= & \m{F}^\m{av}( 1,1,1,1,~\frac{1}{1.5},~~~1,1 ) & \approx & \textbf{0.952};\\
        \quality^{\m{Ld}}_{\m{av}}(E,D_2)= & \m{F}^\m{av}( 1,1,1,1,\frac{2}{2.375},\frac{4}{5} ,1 ) & \approx & 0.949;  \\
        \quality^{\m{Ld}}_{\m{av}}(E,D_3)= & \m{F}^\m{av}( 1,1,\frac{3}{4},1,\frac{1}{1.625},1,1 ) & \approx & 0.909. \\
    \hline
    
        
    \hline
    \quality^{\m{Ld}}_{\m{pr}}(E,D_1)= & \m{F}^\m{pr}(1,1,1,1,~~\frac{1}{1.5},~~1,1 ) & \approx & 0.667; \\
    \quality^{\m{Ld}}_{\m{pr}}(E,D_2)= & \m{F}^\m{pr}(1,1,1,1,\frac{2}{2.375},\frac{4}{5},1 ) & \approx & \textbf{0.674}; \\
    \quality^{\m{Ld}}_{\m{pr}}(E,D_3)= & \m{F}^\m{pr}( 1,1,\frac{3}{4},1,\frac{1}{1.625},1,1 ) & \approx & 0.462. \\
    \hline
    \end{array}$
    }
    \end{center}
    

    \begin{center}
    \resizebox{0.44\textwidth}{!}{ 
    $\begin{array}{lllc}
    \hline
    \quality^{\m{Sd}}_{\m{av}}(E,D_1)= & \m{F}^\m{av}( 1,1,1,1,~~~\frac{1}{9},~~~1,~\frac{1}{2}  ) & \approx & \textbf{0.802}; \\
    \quality^{\m{Sd}}_{\m{av}}(E,D_2)= & \m{F}^\m{av}( 1,0,1,1,~~~\frac{2}{8},~~\frac{9}{10},\frac{1}{2}  ) & \approx & 0.664; \\
    \quality^{\m{Sd}}_{\m{av}}(E,D_3)= & \m{F}^\m{av}( 0,1,\frac{1}{2} ,1,~~\frac{1}{11},~~1,~\frac{2}{3} ) & \approx & 0.608. \\
    \hline
        
    \quality^{\m{Sd}}_{\m{pr}}(E,D_1)= & \m{F}^\m{pr}( 1,1,1,1,~~~\frac{1}{9},~~~1,~\frac{1}{2} ) & \approx & \textbf{0.056}; \\
    \quality^{\m{Sd}}_{\m{pr}}(E,D_2)= & \m{F}^\m{pr}( 1,0,1,1,~~~\frac{2}{8},~~\frac{9}{10},\frac{1}{2}  ) & = & 0; \\
    \quality^{\m{Sd}}_{\m{pr}}(E,D_3)= & \m{F}^\m{pr}( 0,1,\frac{1}{2} ,1,~~\frac{1}{11},~~1,~\frac{2}{3} ) & = & 0. \\
    \hline
    \end{array}$
    }
    \end{center}

To begin with, let us note that there are two possible goals with the output of a quality measure: i) to extract the k-best decodings using the ranking or ii) to extract the ``acceptable" decodings using the numerical values with a threshold.

To extract the best decoding, we can see in bold that according to $\quality^{\m{Ld}}_{\m{av}}$, $D_1$ (a researcher is generally happy) is first, with a better stability score, i.e. the weights of the supports of $D_1$ ($\min = 0.7$) are more appropriate to infer the claim ($\min = 0.7$) than those of $D_2$ ($\min = 0.8$). 
For $\quality^{\m{Ld}}_{\m{pr}}$,
$D_2$ (Bob is loved and often being loved makes people happy) obtains a better score than $D_1$ thanks to a better similarity score and an higher product between similarity and stability values ($\frac{2}{2. 375} \times \frac{4}{5} > \frac{1}{1.5} \times 1$). 
For the quality measures using the strict detailed criteria, $\quality^{\m{Sd}}_{\m{av}}$ and $\quality^{\m{Sd}}_{\m{pr}}$, $D_1$ is the highest scored. 
However, now, if we want to extract the ``acceptable" decodings according to a threshold (e.g., $0.5$), then with $\quality^{\m{Sd}}_{\m{av}}$ the 3 decodings are selected whereas for $\quality^{\m{Sd}}_{\m{pr}}$ no decoding is ``acceptable".
This example shows that for the same set of criteria, aggregation can modify the ranking or drastically change the values.

\section{Conclusion}

Enthymemes are an omnipresent phenomenon, and 
to build systems that can understand them, we need methods to measure the quality of decodings, and thereby optimize the choice of decodings. 
This paper introduces an unexplored research question on the evaluation of enthymeme decoding. 
We propose a generic approach
accepting any weighted logics with an axiomatic framework. 
We investigate 
different quality measures based on aggregation functions and criterion measures, analysed to ensure desirable behaviour. 

To extend our proposal, a formal study of the properties of these quality measures 
is required to guarantee and explain their overall operation. 
The choice of criteria 
can be defined by a user in a context, 
but the numerical parameterisation of these measures and aggregations is not straightforward.
Fortunately, a solution is to learn these configurations from examples. 
Finally, relying on advances in translation of text into logic and the growth of knowledge graphs (interpretable as logical formulae), we plan to apply these quality measures to optimize the generation of decoding from  practical data.

\section*{Acknowledgments}
This work was supported by the French government, managed by the Agence Nationale de la Recherche under the Plan d'Investissement France 2030, as part of the Initiative d'Excellence d'Université Côte d'Azur under the reference ANR-15-IDEX-01.


\bibliography{biblio}
\bibliographystyle{alpha}

\newpage

\section{Appendix: Proofs}

\begin{proof}[\upshape{\textbf{Proof (Proposition \ref{prop:arg})}}]
Let $D = \langle\Delta,\beta\rangle \in \Arg$. 
So $\Delta$ is consistent,
$\Delta{\nci}\beta$ holds, 
and there is no $\Delta'\subset\Delta$ s.t. $\Delta'{\nci}\beta$ holds.
So $\bM(E,D)= \bM'(E,D) = \bM''(E,D) = 1$. 
\end{proof}

\begin{proof}[\upshape{\textbf{Proof (Proposition \ref{prop:sat-coh})}}]
Let $\bL = \tuple{\cW,{\nci},t}$ be a weighted logic, $\forall~E \in \ent$, $\fa D = \la \Delta, \beta \ra, D' = \la \Delta', \beta \ra \in \aArg$, let 
$\mathtt{Nb\_SInc}(E,D) = {\mid\{} \Phi \subseteq \Delta : \Phi \in \Inc , \nexists \Psi \subset \Phi \text{ s.t. }\Psi \in \Inc {\}\mid}$, and 
$\mathtt{Nb\_WInc}(E,D) = {\mid\{} \Phi \subseteq \Delta \cup \Gamma : \Phi \in \Inc , \nexists \Psi \subset \Phi \text{ s.t. }\Psi \in \Inc {\}\mid}$.\\

($\psco{}$) For any $p \in (0,1]$:
\begin{itemize}
    \item if $\Delta$ (resp. $\Delta \cup \Gamma$)  is consistent then $\mathtt{Nb\_SInc}(E,D) = 0$, i.e., ${\mathtt{M}_{\bL p}^\mathtt{psc}}(E,D) = \mymax \big(0, 1 - p \times 0 \big) = 1$ (satisfaction of the axiom Ideal Strong (resp. Weak) Coherence).
    \item if ${\mid\{} \Phi \subseteq \Delta$ (resp. $\Delta \cup \Gamma$) $: \Phi \in \Inc \text{ and } \nexists \Psi \subset \Phi \text{ s.t. } \Psi \in \Inc {\}\mid} \geq 
    {\mid\{} \Phi' \subseteq \Delta '$ (resp. $\Delta' \cup \Gamma$) $ : \Phi' \in \Inc  \text{ and }$   $\nexists \Psi' \subset \Phi' \text{ s.t. }$ $\Psi' \in \Inc  {\}\mid}$  then $\mathtt{Nb\_SInc}(E,D) \geq \mathtt{Nb\_SInc}(E,D')$, i.e., 
    ${\mathtt{M}_{\bL p}^\mathtt{psc}}(E,D) = \mymax \big(0, 1 - p \times \mathtt{Nb\_SInc}(E,D) \big) \leq {\mathtt{M}_{\bL p}^\mathtt{psc}}(E,D') = \mymax \big(0, 1 - p \times \mathtt{Nb\_SInc}(E,D') \big)$ 
    (satisfaction of the axioms Lenient Decreasing Strong and Weak Coherence).
\end{itemize}

($\pwco{}$) For any $p \in (0,1]$:
\begin{itemize}
    \item if $\Delta \cup \Gamma$ is consistent then $\mathtt{Nb\_WInc}(E,D) = 0$, i.e., ${\mathtt{M}_{\bL p}^\mathtt{pwc}}(E,D) = \mymax \big(0, 1 - p \times 0 \big) = 1$ (satisfaction of the axiom Ideal Weak Coherence).
    \item if ${\mid\{} \Phi \subseteq \Delta \cup \Gamma : \Phi \in \Inc \text{ and } \nexists \Psi \subset \Phi \text{ s.t. } \Psi \in \Inc {\}\mid} \geq 
    {\mid\{} \Phi' \subseteq \Delta' \cup \Gamma : \Phi' \in \Inc  \text{ and }$   $\nexists \Psi' \subset \Phi' \text{ s.t. }$ $\Psi' \in \Inc  {\}\mid}$  then $\mathtt{Nb\_WInc}(E,D) \geq \mathtt{Nb\_WInc}(E,D')$, i.e., 
    ${\mathtt{M}_{\bL p}^\mathtt{pwc}}(E,D) = \mymax \big(0, 1 - p \times \mathtt{Nb\_WInc}(E,D) \big) \leq {\mathtt{M}_{\bL p}^\mathtt{pwc}}(E,D') = \mymax \big(0, 1 - p \times \mathtt{Nb\_WInc}(E,D') \big)$ 
    (satisfaction of the axioms Lenient Decreasing Weak Coherence).
\end{itemize}

($\dwco{}$):
\begin{itemize}
    \item if $\Delta \cup \Gamma$ is consistent then $\mathtt{Nb\_WInc}(E,D) = 0$, i.e., ${\mathtt{M}_{\bL}^\mathtt{dwc}}(E,D) = \dfrac{1}{1 + 0} = 1$ (satisfaction of the axiom Ideal Weak Coherence).
    \item if ${\mid\{} \Phi \subseteq \Delta \cup \Gamma : \Phi \in \Inc \text{ and } \nexists \Psi \subset \Phi \text{ s.t. } \Psi \in \Inc {\}\mid} \geq$ (resp. $>$) 
    ${\mid\{} \Phi' \subseteq \Delta' \cup \Gamma : \Phi' \in \Inc  \text{ and }$   $\nexists \Psi' \subset \Phi' \text{ s.t. }$ $\Psi' \in \Inc  {\}\mid}$  then $\mathtt{Nb\_WInc}(E,D) \geq $ (resp. $>$) $ \mathtt{Nb\_WInc}(E,D')$, i.e., 
    ${\mathtt{M}_{\bL}^\mathtt{dwc}}(E,D) = \dfrac{1}{1 + \mathtt{Nb\_WInc}(E,D)} \leq$ (resp. $<$) ${\mathtt{M}_{\bL}^\mathtt{dwc}}(E,D') = \dfrac{1}{1 + \mathtt{Nb\_WInc}(E,D')}$  
    (satisfaction of the axioms Lenient (resp. Strict) Decreasing Weak Coherence).
\end{itemize}

($\dsco{}$):
\begin{itemize}
    \item if $\Delta$ (resp. $\Delta \cup \Gamma$) is consistent then $\mathtt{Nb\_SInc}(E,D) = 0$, i.e., ${\mathtt{M}_{\bL}^\mathtt{dsc}}(E,D) = \dfrac{1}{1 + 0} = 1$ (satisfaction of the axiom Ideal Strong (resp. Weak) Coherence).
    \item if ${\mid\{} \Phi \subseteq \Delta $ (resp. $\Delta \cup \Gamma$) $: \Phi \in \Inc \text{ and } \nexists \Psi \subset \Phi \text{ s.t. } \Psi \in \Inc {\}\mid} \geq$ (resp. $>$) 
    ${\mid\{} \Phi' \subseteq \Delta'$ (resp. $\Delta' \cup \Gamma$) $: \Phi' \in \Inc  \text{ and }$   $\nexists \Psi' \subset \Phi' \text{ s.t. }$ $\Psi' \in \Inc  {\}\mid}$  then $\mathtt{Nb\_WInc}(E,D) \geq $ (resp. $>$) $ \mathtt{Nb\_WInc}(E,D')$, i.e., 
    ${\mathtt{M}_{\bL}^\mathtt{dsc}}(E,D) = \dfrac{1}{1 + \mathtt{Nb\_SInc}(E,D)} \leq$ (resp. $<$) ${\mathtt{M}_{\bL}^\mathtt{dsc}}(E,D') = \dfrac{1}{1 + \mathtt{Nb\_SInc}(E,D')}$  
    (satisfaction of the axioms Lenient (resp. Strict) Decreasing Weak (resp. Strong) Coherence).
\end{itemize}

We can also add that the satisfaction of the Strong Coherence axioms (i.e. $a \rightarrow c$) implies the satisfaction of the Weak Coherence axioms (i.e. $(a \wedge b) \rightarrow c$), given that the condition ``$a \wedge b = \Delta \cup \Gamma$ is consistent" implies ``$a = \Delta$ is consistent", as illustrated in the following example: $(a \wedge b) \vdash a$ and $(a \rightarrow c) \vdash (a \wedge b) \rightarrow c$.

\end{proof}

\begin{proof}[\upshape{\textbf{Proof (Proposition \ref{prop:sat-inf})}}]
Let $\bL = \tuple{\cW,{\nci},t}$ be a weighted logic, $N$ a normalization method on $\cW$, $a \in [0,1]$ be an \textit{acceptable error}, and $V$ be the weight aggregator used in $\bL$. 
Let $E \in \ent$, and $D = \tuple{\Delta,\beta}, D' = \tuple{\Delta',\beta} \in \aArg$.

($\ppi{1}$), we have $\abso{V(\Delta) - V(\beta)} \leq 1$, and:
\begin{itemize}
    \item if $\Flat(\Delta) \vdash \Flat(\beta)$, then $\card{\fCN{N}(\beta) \setminus \fCN{N}(\Delta)} = 0$, therefore 
    $\ppi{1}(E,D) = \mymax \big(0 , 1 - p \times 0 \big) = 1$ (satisfaction of the axiom Ideal Flat Inference)

    \item if $\Delta {\nci} \beta$, then $\card{\fCN{N}(\beta) \setminus \fCN{N}(\Delta)} = 0$, therefore 
    $\ppi{1}(E,D) = \mymax \big(0 , 1 - p \times 0 \big) = 1$ (satisfaction of the axiom Ideal Weighted Inference)

    \item if $\card{\{ f : \Flat(\Delta) \vdash f \text{ and } \Flat(\beta) \vdash f \}} ~ \geq \card{\{ f : \Flat(\Delta') \vdash f \text{ and } \Flat(\beta) \vdash f \}} $, then $\card{\fCN{N}(\beta) \setminus \fCN{N}(\Delta)} \leq \card{\fCN{N}(\beta) \setminus \fCN{N}(\Delta')}$, therefore 
    $\ppi{1}(E,D) \geq \ppi{1}(E,D')$ (satisfaction of the axiom Lenient Increasing Flat Inference)

    \item if $\card{\{ \alpha : \Delta {\nci} \alpha \text{ and } \beta {\nci} \alpha \}} \geq  \card{\{ \alpha : \Delta' {\nci} \alpha \text{ and } \beta {\nci} \alpha \}}$, then $\card{\fCN{N}(\beta) \setminus \fCN{N}(\Delta)} \leq \card{\fCN{N}(\beta) \setminus \fCN{N}(\Delta')}$, therefore 
    $\ppi{1}(E,D) \geq \ppi{1}(E,D')$ (satisfaction of the axiom Lenient Increasing Weighted Inference)
\end{itemize}

($\ppi{<1}$):
\begin{itemize}
    \item if $\Delta {\nci} \beta$, then $\abso{V(\Delta) - V(\beta)} = 0 \leq a$ and $\card{\fCN{N}(\beta) \setminus \fCN{N}(\Delta)} = 0$, therefore 
    $\ppi{1}(E,D) = \mymax \big(0 , 1 - p \times 0 \big) = 1$ (satisfaction of the axiom Ideal Weighted Inference)

    \item if $\card{\{ \alpha : \Delta {\nci} \alpha \text{ and } \beta {\nci} \alpha \}} \geq  \card{\{ \alpha : \Delta' {\nci} \alpha \text{ and } \beta {\nci} \alpha \}}$, then $\abso{V(\Delta) - V(\beta)} = 0 \leq a$ and $\card{\fCN{N}(\beta) \setminus \fCN{N}(\Delta)} \leq \card{\fCN{N}(\beta) \setminus \fCN{N}(\Delta')}$, therefore 
    $\ppi{1}(E,D) \geq \ppi{1}(E,D')$ (satisfaction of the axiom Lenient Increasing Weighted Inference)
\end{itemize}

($\dpi{<1}$):
\begin{itemize}
    \item if $\Delta {\nci} \beta$, then $\abso{V(\Delta) - V(\beta)} = 0 \leq a$ and $\card{\fCN{N}(\beta) \setminus \fCN{N}(\Delta)} = 0$, therefore 
    $\dpi{1}(E,D) = \dfrac{\card{\fCN{N}(\beta)}} {\card{\fCN{N}(\beta)} + 0} = 1$ (satisfaction of the axiom Ideal Weighted Inference)

    \item if $\card{\{ \alpha : \Delta {\nci} \alpha \text{ and } \beta {\nci} \alpha \}} \geq  \card{\{ \alpha : \Delta' {\nci} \alpha \text{ and } \beta {\nci} \alpha \}}$, then $\abso{V(\Delta) - V(\beta)} = 0 \leq a$ and $\card{\fCN{N}(\beta) \setminus \fCN{N}(\Delta)} \leq \card{\fCN{N}(\beta) \setminus \fCN{N}(\Delta')}$, therefore 
    $\dpi{1}(E,D) \geq \dpi{1}(E,D')$ (satisfaction of the axiom Lenient Increasing Weighted Inference)

    \item if $\card{\{ \alpha : \Delta {\nci} \alpha \text{ and } \beta {\nci} \alpha \}} >  \card{\{ \alpha : \Delta' {\nci} \alpha \text{ and } \beta {\nci} \alpha \}}$, then $\abso{V(\Delta) - V(\beta)} = 0 \leq a$ and $\card{\fCN{N}(\beta) \setminus \fCN{N}(\Delta)} < \card{\fCN{N}(\beta) \setminus \fCN{N}(\Delta')}$, therefore 
    $\dpi{1}(E,D) > \dpi{1}(E,D')$ (satisfaction of the axiom Strict Increasing Weighted Inference)
\end{itemize}

($\dpi{1}$), we have $\abso{V(\Delta) - V(\beta)} \leq 1$, and:
\begin{itemize}
    \item if $\Flat(\Delta) \vdash \Flat(\beta)$, then $\card{\fCN{N}(\beta) \setminus \fCN{N}(\Delta)} = 0$, therefore 
    $\dpi{1}(E,D) = \dfrac{\card{\fCN{N}(\beta)}} {\card{\fCN{N}(\beta)} + 0} = 1$ (satisfaction of the axiom Ideal Flat Inference)

    \item if $\Delta {\nci} \beta$, then $\card{\fCN{N}(\beta) \setminus \fCN{N}(\Delta)} = 0$, therefore 
    $\dpi{1}(E,D) = \dfrac{\card{\fCN{N}(\beta)}} {\card{\fCN{N}(\beta)} + 0} = 1$ (satisfaction of the axiom Ideal Weighted Inference)

    \item if $\card{\{ f : \Flat(\Delta) \vdash f \text{ and } \Flat(\beta) \vdash f \}} ~ \geq \card{\{ f : \Flat(\Delta') \vdash f \text{ and } \Flat(\beta) \vdash f \}} $, then $\card{\fCN{N}(\beta) \setminus \fCN{N}(\Delta)} \leq \card{\fCN{N}(\beta) \setminus \fCN{N}(\Delta')}$, therefore 
    $\dpi{1}(E,D) \geq \dpi{1}(E,D')$ (satisfaction of the axiom Lenient Increasing Flat Inference)

    \item if $\card{\{ \alpha : \Delta {\nci} \alpha \text{ and } \beta {\nci} \alpha \}} \geq  \card{\{ \alpha : \Delta' {\nci} \alpha \text{ and } \beta {\nci} \alpha \}}$, then $\card{\fCN{N}(\beta) \setminus \fCN{N}(\Delta)} \leq \card{\fCN{N}(\beta) \setminus \fCN{N}(\Delta')}$, therefore 
    $\dpi{1}(E,D) \geq \dpi{1}(E,D')$ (satisfaction of the axiom Lenient Increasing Weighted Inference)

    \item if $\card{\{ f : \Flat(\Delta) \vdash f \text{ and } \Flat(\beta) \vdash f \}} ~ > \card{\{ f : \Flat(\Delta') \vdash f \text{ and } \Flat(\beta) \vdash f \}} $, then $\card{\fCN{N}(\beta) \setminus \fCN{N}(\Delta)} < \card{\fCN{N}(\beta) \setminus \fCN{N}(\Delta')}$, therefore 
    $\dpi{1}(E,D) > \dpi{1}(E,D')$ (satisfaction of the axiom Strict Increasing Flat Inference)

    \item if $\card{\{ \alpha : \Delta {\nci} \alpha \text{ and } \beta {\nci} \alpha \}} >  \card{\{ \alpha : \Delta' {\nci} \alpha \text{ and } \beta {\nci} \alpha \}}$, then $\card{\fCN{N}(\beta) \setminus \fCN{N}(\Delta)} < \card{\fCN{N}(\beta) \setminus \fCN{N}(\Delta')}$, therefore 
    $\dpi{1}(E,D) > \dpi{1}(E,D')$ (satisfaction of the axiom Strict Increasing Weighted Inference)
\end{itemize}

\end{proof}

\begin{proof}[\upshape{\textbf{Proof (Proposition \ref{prop:sat-min})}}]
\upshape{
Let $\bL = \tuple{\cW,{\nci},t}$ be a weighted logic, $N$ a normalization method on $\cW$, 
$p \in (0,1]$, and $E = \tuple{\Gamma,\alpha} \in \ent, D = \tuple{\Delta,\beta}, D' = \tuple{\Delta',\beta} \in \aArg$.  
Let us recall that from Definition \ref{def:cmin}, $\cinf_{\bL N}(\Delta,\beta) = \lb \Gamma : \Gamma \subseteq N(\Delta) \textrm{ and } \Flat(\Gamma) \vdash \Flat(\beta) \rb.$

($\cmpen{}$):
\begin{itemize}
    \item $\forall p \in (0,1]$, if $\cinf(\Delta,\beta) = \emptyset, \text{ then } \cmpen{p}(E,D) = 1$, otherwise: 
    if $\fa \Phi \subset \Delta, \Flat(\Phi) {\not\vdash} ~ \Flat(\beta)$, then $\cmpen{p}(E,D) = \mymax \big( 0, 1 - p \times \big( |\Delta| - \mymin \lb |\Gamma| : \Gamma \in \cinf(\Delta,\beta) \rb \big) \big) = \mymax \big( 0, 1 - p \times \big( |\Delta| - |\Delta| \big) \big) = 1$ (satisfaction of the axiom Ideal Flat Minimality);
    \item Since weighted inference implies flat inference, using the same reasoning, we also have $\cmpen{p}(E,D) = 1$ (satisfaction of the axiom Ideal Weighted Minimality);

    \item $\forall p \in (0,1]$, if $\cinf(\Delta',\beta) = \emptyset$, or $\Delta'$ is minimal to implies $\beta$, then in both cases $\cmpen{p}(E,D') = 1$, and so for any $\Delta$, $\cmpen{p}(E,D) \leq \cmpen{p}(E,D')$. 
    Additionally in the general cases, if $\card{\{\Gamma : \Gamma \subset \Delta \text{ s.t. } \Flat(\Gamma) {\vdash} ~ \Flat(\beta)\}} \geq  \card{\{\Gamma : \Gamma' \subset \Delta' \text{ s.t. } \Flat(\Gamma) {\vdash} ~ \Flat(\beta) \}}$, then $|\Delta| - \mymin \lb |\Gamma| : \Gamma \in \cinf(\Delta,\beta) \rb  \geq |\Delta'| - \mymin \lb |\Gamma| : \Gamma \in \cinf(\Delta',\beta) \rb$, i.e. $\cmpen{p}(E,D) \leq \cmpen{p}(E,D')$ (satisfaction of the axiom Lenient Decreasing Flat Minimality);

    \item Since weighted inference implies flat inference, using the same reasoning, we also have $\cmpen{p}(E,D) \leq \cmpen{p}(E,D')$ (satisfaction of the axiom Lenient Decreasing Weighted Minimality);
\end{itemize}

($\cmmin{}$):
\begin{itemize}
    \item If $\cinf(\Delta,\beta) = \emptyset$ (i.e., $\Delta = \emptyset$ or $\Flat(\Delta) {\not\vdash} \Flat(\beta)$), $\text{ then } \cmmin{}(E,D) = 1$, otherwise: 
    if $\fa \Phi \subset \Delta, \Flat(\Phi) {\not\vdash} ~ \Flat(\beta)$, and $\cinf(\Delta,\beta) \neq \emptyset$, then $\cmmin{}(E,D) = \dfrac{1}{\card{\cinf(\Delta,\beta)}} = \dfrac{1}{1} = 1$ (satisfaction of the axiom Ideal Flat Minimality);
    \item Since weighted inference implies flat inference, using the same reasoning, we also have $\cmmin{}(E,D) = 1$ (satisfaction of the axiom Ideal Weighted Minimality);

    \item If $\cinf(\Delta',\beta) = \emptyset$ or $\Delta'$ is minimal to implies $\beta$, then in both cases $\cmmin{}(E,D') = 1$, and so for any $\Delta$, $\cmmin{}(E,D) \leq \cmmin{}(E,D')$. 
    Additionally, if $\card{\{\Gamma : \Gamma \subset \Delta \text{ s.t. } \Flat(\Gamma) {\vdash} ~ \Flat(\beta)\}} \geq  \card{\{\Gamma : \Gamma' \subset \Delta' \text{ s.t. } \Flat(\Gamma) {\vdash} ~ \Flat(\beta) \}}$, then $\card{\cinf(\Delta,\beta)} \geq \card{\cinf(\Delta',\beta)}$, i.e., $\cmmin{}(E,D) = \dfrac{1}{\card{\cinf(\Delta,\beta)}} \leq \dfrac{1}{\card{\cinf(\Delta',\beta)}} = \cmmin{}(E,D')$ (satisfaction of the axiom Lenient Decreasing Flat Minimality);
    \item Since weighted inference implies flat inference, using the same reasoning, we also have $\cmmin{}(E,D) \leq \cmmin{}(E,D')$ (satisfaction of the axiom Lenient Decreasing Weighted Minimality);

    \item If $\cinf(\Delta',\beta) = \emptyset$ or $\Delta'$ is minimal to implies $\beta$, then in both cases $\cmmin{}(E,D') = 1$, and if $\card{\{\Gamma : \Gamma \subset \Delta \text{ s.t. } \Flat(\Gamma) {\vdash} ~ \Flat(\beta)\}} >  \card{\{\Gamma : \Gamma' \subset \Delta' \text{ s.t. } \Flat(\Gamma) {\vdash} ~ \Flat(\beta) \}}$ then $\cinf(\Delta,\beta) \neq \emptyset$ and $\Delta$ is not minimal to implies $\beta$ (otherwise the cardinality are both equal to 0 and so there is no $>$), i.e., $\card{\cinf(\Delta,\beta)} > 1$, thus $\cmmin{}(E,D) < 1$ ; hence $\cmmin{}(E,D) < \cmmin{}(E,D')$. 
    Additionally in the general cases, if $\card{\{\Gamma : \Gamma \subset \Delta \text{ s.t. } \Flat(\Gamma) {\vdash} ~ \Flat(\beta)\}} >  \card{\{\Gamma : \Gamma' \subset \Delta' \text{ s.t. } \Flat(\Gamma) {\vdash} ~ \Flat(\beta) \}}$, then $\card{\cinf(\Delta,\beta)} > \card{\cinf(\Delta',\beta)}$, i.e., $\cmmin{}(E,D) = \dfrac{1}{\card{\cinf(\Delta,\beta)}} < \dfrac{1}{\card{\cinf(\Delta',\beta)}} = \cmmin{}(E,D')$ (satisfaction of the axiom Strict Decreasing Flat Minimality);
    \item Since weighted inference implies flat inference, using the same reasoning, we also have $\cmmin{}(E,D) \leq \cmmin{}(E,D')$ (satisfaction of the axiom Strict Decreasing Weighted Minimality);
\end{itemize}

}
\end{proof}

\begin{proof}[\upshape{\textbf{Proof (Proposition \ref{prop:sat-sim})}}]
Let $\bL = \tuple{\cW,{\nci},t}$ be a weighted logic and $N$ a normalization method on $\cW$.
Let $x, y \in (0,+\infty)$ and $E = \tuple{\Gamma,\alpha} \in \ent, D = \tuple{\Delta,\beta}, D' = \tuple{\Delta',\beta} \in \aArg$.
From Definition~\ref{def:extension-tversky-measure}, we have:
$$\m{Tve}_{N}(\Gamma,\Delta, x, y) = 
 			\left\{
 			\begin{array}{l l}
 			1 & \textrm{if } \Gamma = \Delta = \emptyset;\\
 			\dfrac{a}{a + x \times  b + y \times c} & \textrm{otherwise,}\\
 			\end{array}
 			\right.$$
    where $a= \card{N(\Gamma) \cap N(\Delta)}$, $b = \card{N(\Gamma) \setminus N(\Delta)}$, and $c = \card{N(\Delta) \setminus N(\Gamma)}$; 
    $a'= \card{N(\Gamma) \cap N(\Delta')}$, \linebreak 
    $b = \card{N(\Gamma) \setminus N(\Delta')}$, and $c = \card{N(\Delta') \setminus N(\Gamma)}$.
    
It is then straightforward to check the following:
\begin{itemize}
    \item Lenient increasing N-similarity: $\text{if } a \geq a', b = b', c = c',\text{ then } \bM(E,D) \geq \bM(E,D').$ 

    \item Strict increasing N-similarity: $\text{if } a > a', b = b', c = c',\text{ then } \bM(E,D) > \bM(E,D').$ 

    \item Lenient decreasing N-similarity: $\text{if } a = a', b \geq b', c \geq c',\text{ then } \bM(E,D) \leq \bM(E,D').$ 

    \item Strict decreasing N-similarity: $\text{if } a = a'$, and $(b > b', c \geq c')$ or $(b \geq b', c > c')$, $\text{then } \bM(E,D) < \bM(E,D').$ 
\end{itemize}

\end{proof}

\begin{proof}[\upshape{\textbf{Proof (Proposition \ref{prop:sat-link})}}]
The case of $\cmbl{}$ is straightforward.
We turn to $\cmtvetve{\m{jac}}$, $\cmtvetve{\m{dic}}$, $\cmtvetve{\m{sor}}$, $\cmtvetve{\m{and}}$, and $\cmtvetve{\m{ss2}}$.
By Definition \ref{def:extension-tversky-measure}, if $a = 0$, then all Tversky measures are equal to $0$.
By Definition $\ref{def:preserv}$, thanks to the product between the two Tversky similarity measures, if the supports have no intersection (or the claims are different), then the complete Tversky preservation measure return $0$.
\end{proof}

\begin{proof}[\upshape{\textbf{Proof (Proposition \ref{prop:sat-cgran})}}]
Let $\bL = \tuple{\cW,{\nci},t}$ be a weighted logic and $N$ a normalization method on $\cW$.
Let $E = \tuple{\Gamma,\alpha} \in \ent, D = \tuple{\Delta,\beta}, D = \tuple{\Delta',\beta} \in \aArg$.
From Def. \ref{def:graCon}:
\begin{itemize}
    \item $\cmcd{}(E,D) > \cmcd{}(E,D')$ iff \\ $\card{N(\Delta) \setminus N(\Gamma)} < \card{N(\Delta') \setminus N(\Gamma)}$.

    \item If $s \in \mathbb{N}^+$, and $p \in (0,1]$, then: $\cmcp{}(E,D) \geq \cmcd{}(E,D')$ iff $\card{N(\Delta) \setminus N(\Gamma)} < \card{N(\Delta') \setminus N(\Gamma)}$.
\end{itemize}

\end{proof}

\begin{proof}[\upshape{\textbf{Proof (Proposition \ref{prop:sat-dgran})}}]
Let $\bL = \tuple{\cW,{\nci},t}$ be a weighted logic and $N$ a normalization method on $\cW$.
Let $E = \tuple{\Gamma,\alpha} \in \ent, D = \tuple{\Delta,\beta}, D' = \tuple{\Delta',\beta} \in \aArg$. 
From Def. \ref{def:graDet}:
\begin{itemize}
    \item $\cmdg{}(E,D) < \cmdg{}(E,D')$ iff \\ $\card{N(\Delta) \setminus N(\Gamma)} < \card{N(\Delta') \setminus N(\Gamma)}$.

    \item If $s \in \mathbb{N}^+$ and $p \in (0,1]$, then: $\cmpg{}(E,D) \leq \cmpg{}(E,D')$ iff $\card{N(\Delta) \setminus N(\Gamma)} < \card{N(\Delta') \setminus N(\Gamma)}$.
\end{itemize}
\end{proof}

\begin{proof}[\upshape{\textbf{Proof (Proposition \ref{prop:sat-wstab})}}]

Let $\bL = \tuple{\cW,{\nci},t}$ be a weighted logic, and  
$A \in \aArg, B_1 = \tuple{\Delta_1,\beta_1}, B_2 = \tuple{\Delta_2,\beta_2} \in \aArg$. 
Let $a, u \in [0,1]$ such that $a < u $, and $Err_i = \abso{\mymin[\Wei(\Delta_i)] - \Wei(\beta_i)}$. 
From Definition \ref{def:stabC}:
\begin{itemize}
    \item Ideal Stability: $\text{if } \mymin(\Wei(\Delta_1)) = \Wei(\beta_1)$ then  $\cmsd{}(E,D_1) =  1 - 0 = \cmld{}(E,D_1) = 1 $, because $Err_1 = 0 \leq a$, $\forall a \in [0,1]$.
    \item Lenient Decreasing Stability: $\text{if } Err_1 \leq Err_2$, then, by definition, $\cmsd{}(E,D_1) \geq \cmsd{}(E,D_2)$ and $\cmld{}(E,D_1) \geq \cmld{}(E,D_2)$ because either the value stops increasing and falls to $0$ if it reaches the maximum error, or the value stops decreasing if it reaches the error tolerance and otherwise it varies according to the difference;
    \item Strict Decreasing Stability: $\text{if } Err_1 <  Err_2$, then $\cmsd{}(E,D_1) > \cmsd{}(E,D_2),$ by definition (i.e., $1-Err_i$).
\end{itemize}

\end{proof}

\end{document}